\documentclass[10pt,reqno]{amsart}

\usepackage{amsfonts, amsmath,amssymb,array,latexsym, hyperref, amscd}
\usepackage{caption}
\usepackage{color}
\usepackage{fancyhdr,a4wide}
\usepackage[a4paper, total={6in, 8in}]{geometry}

\usepackage{graphicx}
\usepackage{subfigure}

\usepackage{float}

\usepackage{amsmath}
\usepackage{amssymb}
\usepackage{mathtools}
\usepackage{amsthm}

\theoremstyle{plain}
\newtheorem{theorem}{Theorem}[section]
\newtheorem{pt}{Practical Takeaway}[section]
\newtheorem{proposition}[theorem]{Proposition}
\newtheorem{lemma}[theorem]{Lemma}
\newtheorem{corollary}[theorem]{Corollary}
\theoremstyle{definition}

\theoremstyle{remark}

\newcommand{\mP}{\mathcal P}
\newcommand{\mG}{\mathcal G}
\newcommand{\mN}{\mathcal N}
\newcommand{\tr}{\mathrm{tr}}
\newcommand{\E}[1]{\mathbb E\left[#1\right]}
\newcommand{\norm}[1]{\left|\left|#1\right|\right|}
\newcommand{\lr}[1]{\left(#1\right)}
\newcommand{\set}[1]{\left\{#1\right\}}
\newcommand{\abs}[1]{\left|#1\right|}
\newcommand{\R}{\mathbb R}

\newcommand{\twomat}[4]{\lr{\begin{array}{cc} #1 & #2 \\ #3 &#4\end{array}}}

\title[Initialization and Architecture in GNNs]{Principles for Initialization and Architecture Selection in Graph Neural Networks with ReLU Activations}

\author{Gage DeZoort}
\address{Department of Physics\\Princeton University}
\email{jdezoort@princeton.edu}

\author{Boris Hanin}
\address{Department of Operations Research and Financial Engineering\\
Princeton University}
\email{bhanin@princeton.edu}

\begin{document}

\maketitle

\begin{abstract}
This article derives and validates three principles for initialization and architecture selection in finite width graph neural networks (GNNs) with ReLU activations. First, we theoretically derive what is essentially the unique generalization to ReLU GNNs of the well-known He-initialization. Our initialization scheme guarantees that the average scale of network outputs and gradients remains order one at initialization. Second, we prove in finite width vanilla ReLU GNNs that oversmoothing is unavoidable at large depth when using fixed aggregation operator, regardless of initialization. We then prove that using \textit{residual aggregation operators}, obtained by interpolating a fixed aggregation operator with the identity, provably alleviates oversmoothing at initialization. Finally, we show that the common practice of using residual connections with a fixup-type initialization provably avoids correlation collapse in final layer features at initialization. Through ablation studies we find that using the correct initialization, residual aggregation operators, and residual connections in the forward pass significantly and reliably speeds up early training dynamics in deep ReLU GNNs on a variety of tasks. 

\end{abstract}

\section{Introduction}\label{sec:intro}
Graph neural networks (GNNs) have emerged in recent years as flexible parametric models for learning from graph-based data~\cite{zhou2020graph,wu2020comprehensive,bronstein2021geometric,musaelian2023learning}. The forward pass in a GNN intertwines the usual feed-forward structure of a neural network with an aggregation step that makes features at neighboring graph vertices more similar (see \S \ref{sec:def}). As with most neural network architectures, however, using GNNs effectively in practice requires carefully choosing a number of architectural and optimization hyperparameters. 
In this article, we blend ideas from spectral graph theory with analysis in the style of deep information propagation \cite{poole2016exponential,schoenholz2016deepinformation,roberts2022principles} to develop practical and principled  initialization schemes and GNN architectures that help take the guesswork out of making GNNs that are reliably trainable to large depth. Specifically, we explain how to provably avoid three common failure modes of GNN training at initialization:\\

\begin{itemize}
    \item[\textbf{(1)}] 
    \textbf{Exponential outputs.} 
    This failure mode occurs when model outputs grow or decay exponentially with network depth, leading to significant slowdowns in early training dynamics (see Figure \ref{fig:evgp-bad}).\footnote{Due to the homogeneity of ReLU networks, exponential growth/decay of network outputs is equivalent to exponential growth/decay of input-output Jacobians since the model output at an input $x$ is simply the inner product between $x$ and the input-output Jacobian evaluated at $x$. As we explain in \S \ref{sec:evgp} this failure mode is also closely connected to the exploding/vanishing gradient problem, which concerns the gradients of the model output with respect to its parameters.} We derive a simple generalization to GNNs of the well-known He-Initialization~\cite{he2015delving} which is sufficient for avoiding exponential growth/decay of model outputs in vanilla ReLU GNNs. See Theorems \ref{thm:He-init-inf} and \ref{thm:He-init} and Figures \ref{fig:evgp-bad}, \ref{fig:init}.\\
    
    
    \item[\textbf{(2)}] \textbf{Oversmoothing.} 
    This well-known failure mode, which is specific to GNNs, occurs when features in the final layer are almost entirely concentrated in the top eigenspace of the aggregation operator \cite{li2018deeper}. When oversmoothing occurs, it often results in poor performance in tasks that must differentiate neighboring vertices (see Figure \ref{fig:os-bad}). We prove that oversmoothing necessarily occurs at large depth in randomly initialized ReLU GNNs that do not have residual connections and use a fixed aggregation operator. In contrast, we prove that \textit{residual aggregation operators} (see \eqref{eq:p-res-def}), obtained by interpolating a given aggregation operator with the identity, provably mitigate oversmoothing in fully connected ReLU GNNs, even without skip connections. This builds on prior empirical work, such as \cite{chen2020simple}. See Theorems \ref{thm:os-inf}, \ref{thm:res-os-inf}, and \ref{thm:oversmoothing} as well as Practical Takeaway \ref{pt:res-ag} and Figures \ref{fig:ablation-performance}, \ref{fig:ablation-time}, \ref{fig:os-bad} for experimental validation that residual aggregation operators improve early training dynamics.\\
    
    \item[\textbf{(3)}] \textbf{Correlation Collapse.} 
    This failure mode occurs when the final hidden layer features are highly correlated across graph vertices. Though related, this failure mode is distinct from oversmoothing since the final layer features may be highly correlated without having low frequency (i.e. being oversmoothed). We give a simple explanation for why residual connections with fixup-type initialization \cite{zhang2019fixup} avoid correlation collapse and show empirically that they enhance trainability in deep ReLU GNNs. Again, such ideas were studied empirically in prior work such as \cite{chen2020simple}. See Practical Takeaway \ref{pt:res-2} and Figures \ref{fig:ablation-performance}, \ref{fig:ablation-time}, \ref{fig:cc-bad}. \\
\end{itemize} 

\begin{figure}
    \centering
    \includegraphics[width=0.7
    \textwidth]{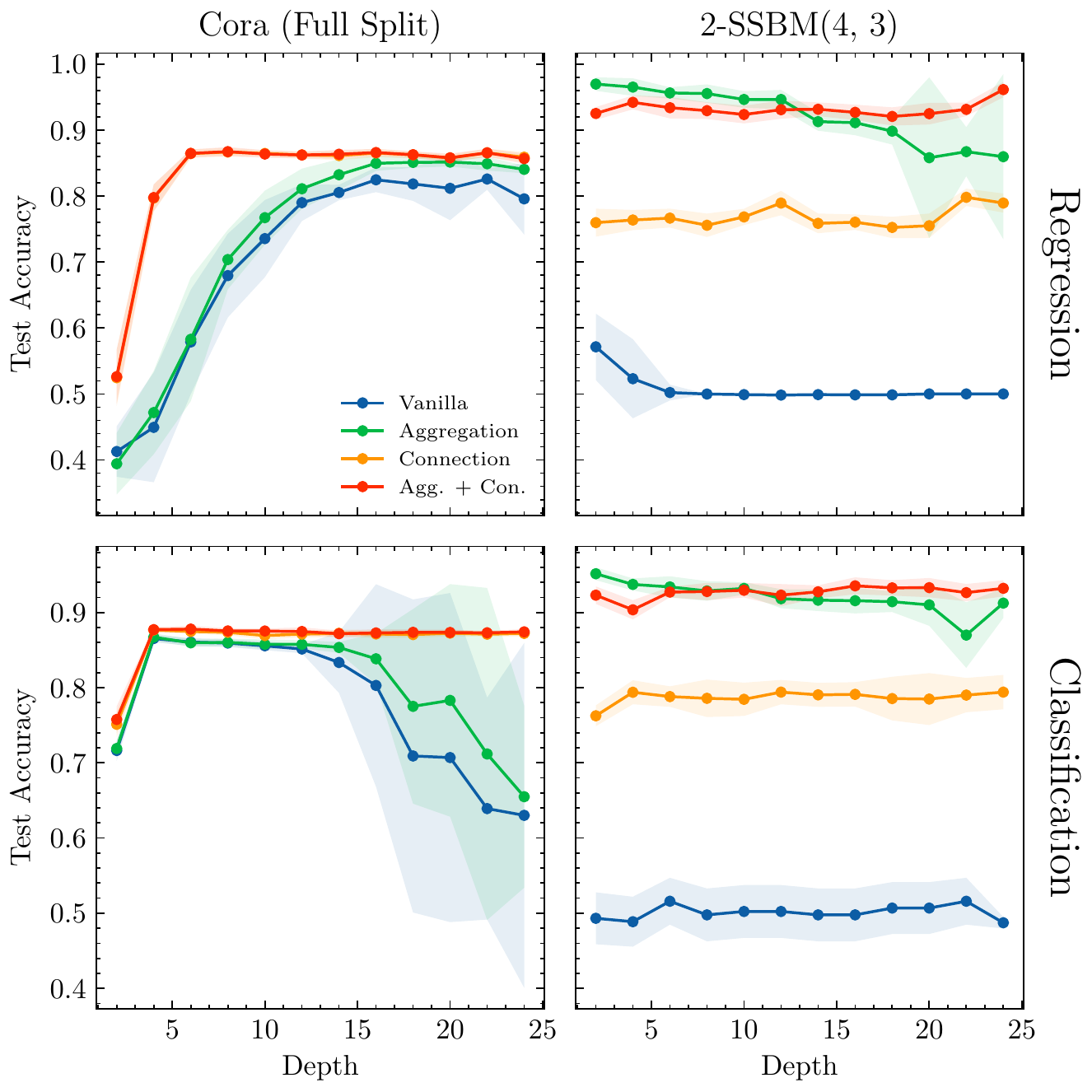}
    \caption{{\footnotesize Classification and regression test accuracy on Cora and two-class stochastic block model (2-SSBM) graph using GNNs of four kinds: vanilla (blue), with residual aggregation operators $(1-t)I + t \mP$ (green), with residual connections (orange), and with both residual aggregation operators and residual connections (red). Symmetric degree-normalized adjacency matrices with self-loops $\mP$ are the base aggregation operators. Layers widths are $128$ for Cora and $64$ for SSBMs. Training is by GD for a maximum of 1500 steps and early stopping with 250 step patience. Initial learning rate is dropped by 0.25 after 200 steps. At each depth, 15 iid initializations with weight variance $C_W^{(\ell)}=2.25/(n_{\ell-1}\lambda_1(\mathcal{P})^2)$ are generated for each set of hyperparameters: $(\mathrm{lr})$ for vanilla GNNs, $(\mathrm{lr}, t)$ for GNNs with residual aggregations, $(\mathrm{lr}, \beta)$ for GNNs with residual connections, and $(\mathrm{lr}, t, \beta)$ for GNNs with both residual aggregations and connections. We scan over $\mathrm{lr}\in\{0.1, 0.05, 0.01, 0.005, 0.001\}$ and $t,\beta\in\{0.1,0.2,0.4,0.6,0.8,0.9\}$. The hyperparameter configuration with the best average performance is reported at each depth. The performance of an individual model is measured as its test set accuracy at the epoch with maximum validation set accuracy. These performances are averaged for each set of hyperparameters, and the set of hyperparameters with the maximum average validation set accuracy is selected. We then report the corresponding average performance on the test set.  
    }}
    \label{fig:ablation-performance}
\end{figure}

Taken together, our theoretical and empirical results suggest that by choosing correct initialization schemes that avoid exponential outputs/gradients as in (1), using residual aggregation operators that avoid oversmoothing as in (2), and inserting residual connections with fixup-type initialization to avoid correlation collapse as in (3) GNNs quickly and reliably begin training, even at large depth. This is illustrated in the ablation studies summarized in Figures \ref{fig:ablation-performance} and \ref{fig:ablation-time}. 

\subsection*{Acknowledgements} GD gratefully acknowledges funding from DOE grant DE-SC0007968. BH gratefully acknowledges funding from NSF CAREER grant DMS-2143754 and NSF grants DMS-1855684, DMS-2133806 as well as an ONR MURI on Foundations of Deep Learning.

\subsection*{Outline} The rest of this article is organized as follows. First, in \S  \ref{sec:exp-inf} we define GNNs, present our main experimental results, and provide an informal discussion of our theoretical contributions. We then give in \S  \ref{sec:lit-rev} an overview of related literature. We turn in \S \ref{sec:results} to precise statements of our main theorems. This is followed by a description in \S \ref{sec:exp} of our experimental setup. Finally, the proofs of our theoretical results are detailed in \S \ref{sec:proofs}.

\begin{figure}
    \centering
    \includegraphics[width=0.7\textwidth]{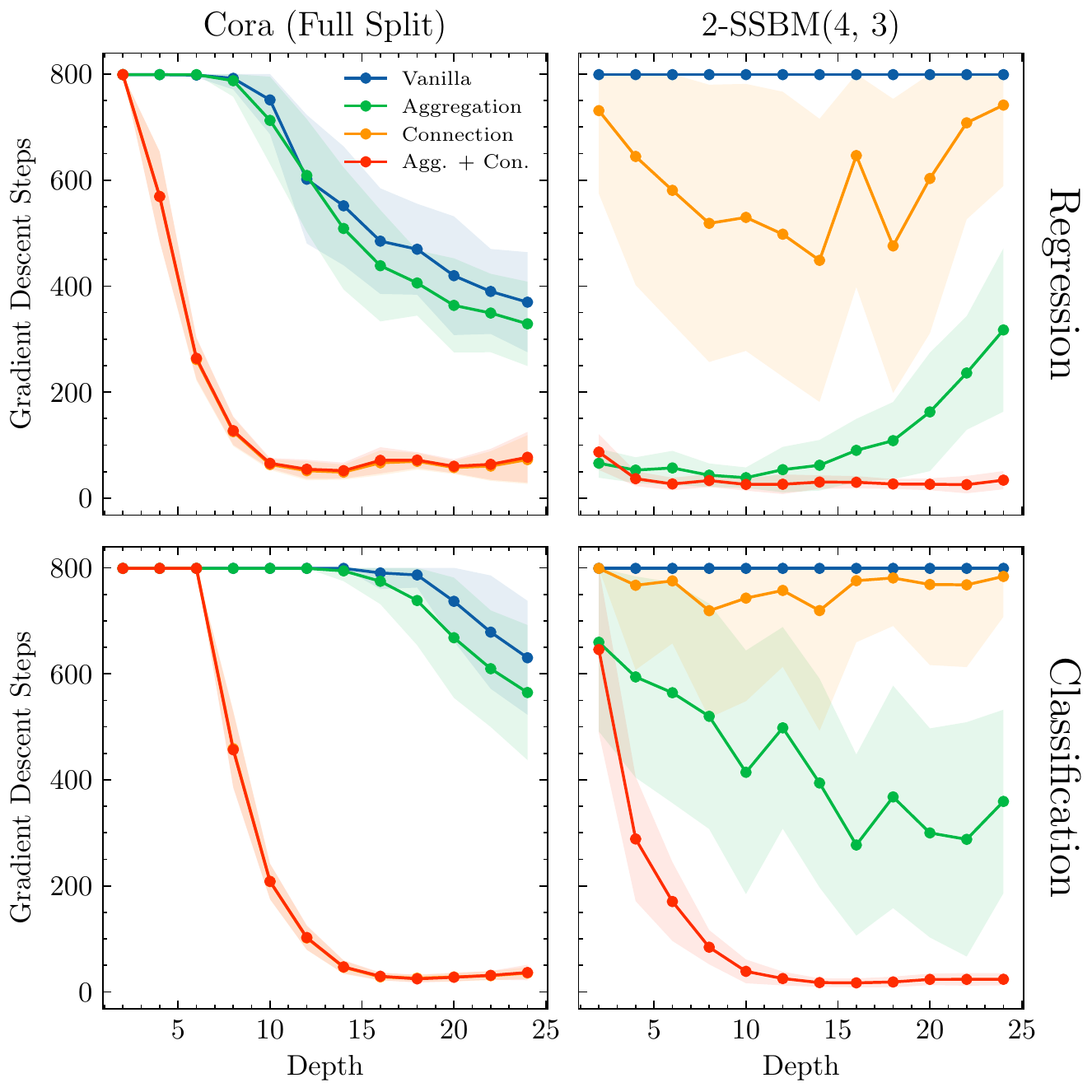}
    \caption{{\footnotesize Number of gradient descent steps with a constant learning rate required on Cora and a two-class stochastic block model (2-SSBM) graph to reach the training accuracy of the best linear classifier using GNNs of four kinds: vanilla (blue), with residual aggregation operators $(1-t)I + t \mP$ (green), with residual connections (orange), and with both residual aggregation operators and residual connections (red). Symmetric degree-normalized adjacency matrices with self-loops $\mP$ are the base aggregation operators. Layers widths are $128$ for Cora and $64$ for SSBMs. Training is by GD for a maximum of 800 steps. At each depth, 15 iid initializations with weight variance $C_W^{(\ell)}=2.25/(n_{\ell-1}\lambda_1(\mathcal{P})^2)$ are generated for each set of hyperparameters: $(\mathrm{lr})$ for vanilla GNNs, $(\mathrm{lr}, t)$ for GNNs with residual aggregations, $(\mathrm{lr}, \beta)$ for GNNs with residual connections, and $(\mathrm{lr}, t, \beta)$ for GNNs with both residual aggregations and connections. The learning rates are kept at a constant $0.05$ for Cora regression,  $0.005$ for Cora classification, $0.01$ for 2-SSBM regression, and $0.001$ for 2-SSBM classification. The hyperparameter set with the fastest average training time is reported at each depth for residual strengths $t,\beta\in\{0.1,0.2,0.4,0.6,0.8,0.9\}$.
}}
\label{fig:ablation-time}
\end{figure}

\section{Experiments and Informal Overview of Theorems}\label{sec:exp-inf}
We present in \S \ref{sec:evgp-inf} - \S \ref{sec:ntk-inf} informal statements of our theoretical results, highlights of our key empirical findings, and summaries our main practical takeaways. Before doing so, we introduce in \S \ref{sec:def} our notation and assumptions. 

\subsection{Definitions and Assumptions}\label{sec:def}
We analyze depth $L$ fully connected GNNs with input dimension $n_0$, output dimension $n_{L+1}$, hidden layer widths $n_1,\ldots, n_L$, and ReLU activations. By definition\footnote{We study only vanilla GNNs that have vertex features but not edge features. We also set biases to zero.}, an input to such a network is a matrix $x\in \R^{\abs{V}\times n_0}$, which assigns to every vertex $v\in V$ on a  graph $\mG = \lr{V,E}$ with $\abs{V}$ vertices a vector of features $x_v\in \R^{n_0}$. The forward pass\footnote{We describe the forward pass for node regression tasks. For graph or node classification, for instance, there may be additional pooling and/or softmax layers.}
\[
x\in \R^{\abs{V}\times n_0}\quad \mapsto\quad z^{(L+1)}(x;\theta)\in \R^{\abs{V}\times n_{L+1}}
\]
is computed by defining pre-activations 
\[
z^{(\ell)}(x)=\lr{z_{v;i}^{(\ell)}(x),\quad v\in V,\, 1\leq i \leq n_\ell}\in \R^{\abs{V}\times n_\ell}
\]
in layer $\ell$ through the recursion
\begin{equation}\label{eq:z-def}
z^{(\ell+1)}(x) = \begin{cases}
 \mP^{(\ell)}\sigma\lr{z^{(\ell)}(x)}W^{(\ell+1)},&\quad \ell \geq 1\\
x W^{(1)},&\quad \ell =0
\end{cases}.
\end{equation}
Here, $\sigma$ is the ReLU applied componentwise, $\mP^{(\ell)}$ is a symmetric $\abs{V}\times\abs{V}$ matrix that we refer to as the aggregation operator in layer $\ell$, and $W^{(\ell+1)}\in \R^{n_{\ell}\times n_{\ell+1}}$ is a matrix of weights. In components, the recursion \eqref{eq:z-def} with $\ell\geq 1$ takes the form
\[
z_{v;i}^{(\ell+1)}(x) = \sum_{u\in V}\sum_{j=1}^{n_{\ell}} \mP_{vu}^{(\ell)} W_{ij}^{(\ell+1)}\sigma\lr{z_{u;j}^{(\ell)}(x)}.
\]
At initialization we take independent centered Gaussian weights
\begin{equation}\label{eq:rand-def}
    W_{ij}^{(\ell)}\sim \mN\lr{0,C_W^{(\ell)}},\qquad i=1,\ldots, n_{\ell+1},\, j=1,\ldots, n_\ell.
\end{equation}
One of our key theoretical and practical contributions, explained in \S \ref{sec:evgp-inf}, is a principled way to choose the weight variances $C_W^{(\ell)}$. We will assume that for every $\ell\geq 1$, the aggregation operators satisfy the following three conditions:
\begin{align}
    \label{eq:P-as1}&\mP^{(\ell)} \text{ is a symmetric}\\
    \label{eq:P-as2}&\mP^{(\ell)} \text{ has non-negative entries}\\ 
    \label{eq:P-as4}&\mP^{(\ell)} \text{ have the same top eigenspace for all $\ell$}
\end{align} 
For example, we may take $\mP^{(\ell)}$ to be a symmetrically normalized adjacency matrix (potentially with self-loops), as is commonly done in practice. As a consequence of \eqref{eq:P-as4}, we may write 
\[
\Pi_1:=\text{ projection onto the top eigenspace of $\mP^{(\ell)}$}.
\]
Finally, for any non-zero network input $x\in \R^{\abs{V}\times n_0}$ we define the \textit{oversmoothing ratio} after $\ell$ layers to be
\begin{equation}\label{eq:r-def}
r^{(\ell)} :=  \E{\frac{1}{n_\ell}\sum_{i=1}^{n_\ell}\norm{\Pi_1 z_i^{(\ell)}(x)}^2}\bigg/\E{\frac{1}{n_\ell}\sum_{i=1}^{n_\ell}{\norm{z_i^{(\ell)}(x)}^2}},    
\end{equation}
where $\norm{\cdot}$ is the $\ell_2$-norm and $\E{\cdot}$ denotes the average with respect to the random weights \eqref{eq:rand-def}. By the homogeneity of ReLU nets and the rotation invariance of the Gaussian, $r^{(\ell)}$ is independent of the input $x$. The oversmoothing ratio  measures the proportion of the $\ell_2$ norm of the pre-activations in layer $\ell$ contained  in the top eigenspace of the aggregation operators $\mP^{(\ell)}$.

\subsection{Avoiding Exponential Outputs}\label{sec:evgp-inf} 
Let us consider a fully connected depth $L$ ReLU GNN on a graph $\mathcal G = \lr{V,E}$. Given a vertex $v\in V$ and a network input $x$, we measure the relative size of the corresponding network outputs using
\begin{equation}
    \text{output distortion}:=\frac{1}{n_L}\norm{z_v^{(L)}(x)}^2\bigg/\frac{1}{n_0}\norm{x_v}^2.
\label{eq:output-distortion}
\end{equation}\\
The exponential outputs failure mode occurs when  output distortion grows or decays exponentially in network depth. To what extent this occurs is practically important since, as we show in Figure \ref{fig:evgp-bad}, large values of the output distortion significantly slow down in learning with vanilla ReLU GNNs at large depth (see also \cite{poole2016exponential,schoenholz2016deepinformation,xiao2018dynamical} for similar conclusions for other architectures). 

\begin{figure}
  \centering              
    \includegraphics[width=0.7\linewidth]{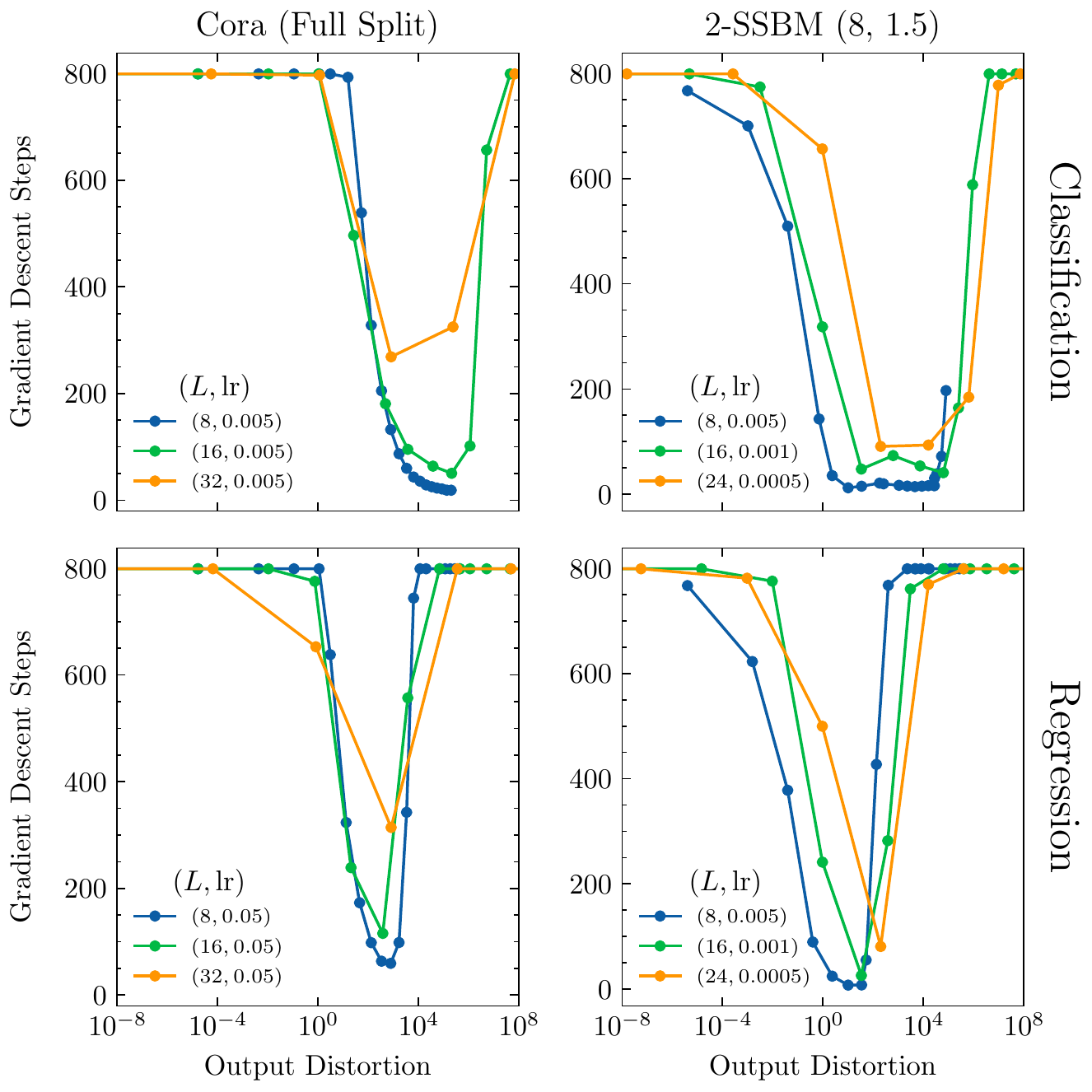} \\
  \caption{{\footnotesize 
  Number of GD steps with constant step size lr on Cora and a two-class stochastic block model (2-SSBM) graph to reach training accuracy of best linear classifier for GNNs with residual aggregation operators $(1-t)I + t \mP$. Here, $\mP$ is the symmetric degree-normalized adjacency matrix with self-loops. The learning task is treated as either classification (cross-entropy loss) or regression (MSE with one-hot labels). Output distortion is varied by setting $C_W^{(\ell)}=C_W/(n_{\ell-1}\lambda_1(\mathcal{P})^2)$ and scanning over $C_W\in\{0.5, 1.0, 1.5, \ldots,10.0\}$ for each depth $L$. We scan over $t\in \{0.2, 0.4, 0.6, 0.8\}$ at each setting of $(L, C_W)$ to avoid oversmoothing and choose the fastest training value averaged over 25 random initializations. The corresponding output distortion is computed as an empirical average over all vertices and each random initialization. All models are allowed to train for a maximum of 800 steps. For settings of ($L$, $C_W$, $t$) that never reached the training accuracy threshold we conservatively report 800 steps.
  }}
  \label{fig:evgp-bad}
\end{figure}

\begin{figure}
  \centering              
    \includegraphics[width=.7\linewidth]{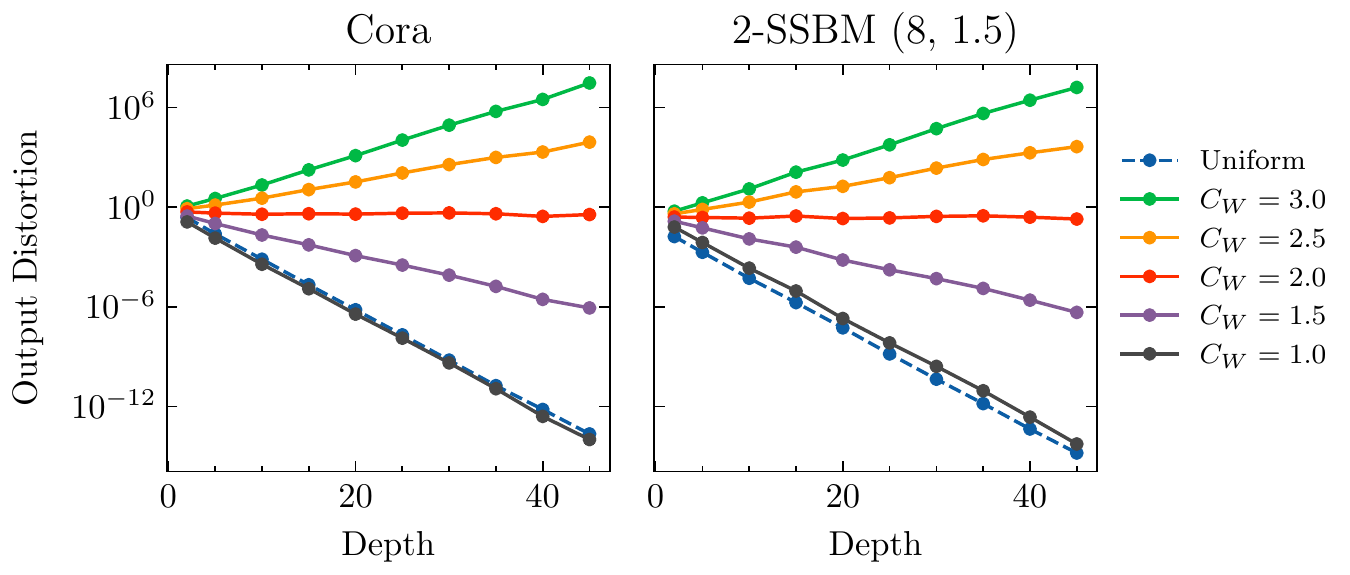} \\
  \caption{\footnotesize Preactivation magnitudes in the last hidden layer plotted against depth for both the default $\textsc{PyTorch Geometric}$ uniform initialization scheme and Gaussian initialization schemes with various values of $C_W^{(\ell)}=C_W/n_{\ell-1}\lambda_1(\mathcal{P})^2$.  For Cora (left) and a 2-SSBM with exact recoverability (right), we generate 100 random GNNs at varying depths and initialization schemes. In all cases we use symmetric degree-normalized adjacency matrices with self-loops as aggregation operators and hidden layer widths are fixed to  $128$ for Cora and $64$ for SSBM. Vertex features from each graph are L2-normalized such that $||x_v||^2=1$ for all vertices $v$ and propagated through each network. We report the output distortion~\eqref{eq:output-distortion} over all graph vertices and random initializations.}
  \label{fig:init}
\end{figure}

\begin{figure}
  \centering              
    \includegraphics[width=.7\linewidth]{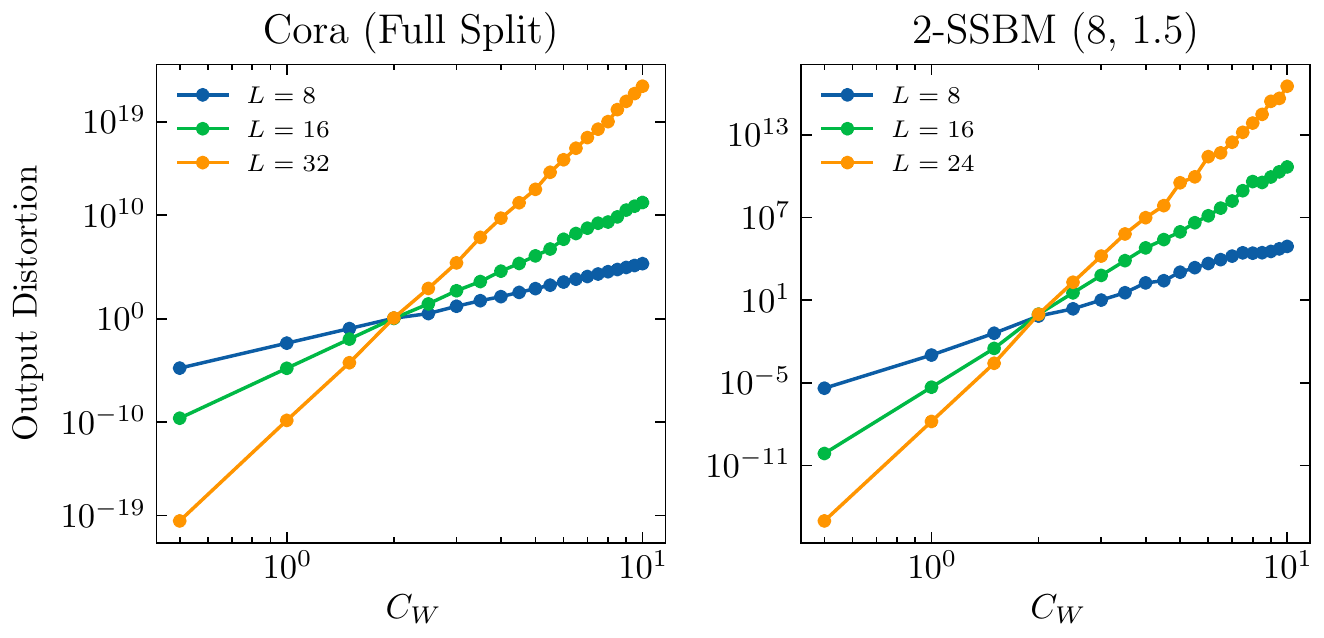} \\
  \caption{\footnotesize For each depth $L$ model reported in the classification portion of Figure \ref{fig:evgp-bad}, we plot the output distortion at initialization as a function of the weight variance scale $C_W$.}
  \label{fig:evgp-fix}
\end{figure}


As shown in Figure \ref{fig:init}, moreover, many commonly used initialization schemes lead to exponentially growing or decaying network outputs (e.g. the uniform and $C_W=1$ lines). However, Figure \ref{fig:init} also suggests that there exists a unique value for the weight variance scale $C_W$ that seems to mitigate this failure mode at large depth. To better understand this, recall that in fully connected non-graph ReLU networks the widely used He-initialization 
\begin{equation}\label{eq:He-intro}
\text{weight variance} ~=~2/\text{fan-in}
\end{equation}
derived in \cite{he2015delving} ensures that in sufficiently wide networks at initialization outputs and gradients neither grow nor decay exponentially with network depth (for a discussion of what constitutes sufficiently wide see \cite{hanin2018neural,hanin2018start}). 

However, this conclusion does not hold for GNNs. Indeed, our first theoretical result shows that the right generalization of He-initialization to ReLU GNNs must take into account the spectral radius of the message passing operator:

\begin{theorem}[Informal Consequence of Theorem \ref{thm:He-init}]\label{thm:He-init-inf} Consider a fully connected ReLU GNN with no residual connections and a fixed aggregation operator $\mathcal P$ satisfying \eqref{eq:P-as1}-\eqref{eq:P-as4}. At initialization, the mean squared activations and gradients neither grow nor decay exponentially with network depth if
\begin{equation}\label{eq:He-gen-intro}
\text{weight variance} ~=~\frac{2}{\text{fan-in}\times \lambda_1(\mP)^2},
\end{equation}
where $\lambda_1(\mP)$ is the largest eigenvalue of $\mP$. 
\end{theorem}

In practice we find that the initialization \eqref{eq:He-gen-intro} is key to allowing vanilla ReLU GNNs to be trainable at fairly large depth (see Figures \ref{fig:evgp-bad},  \ref{fig:init}). Indeed, the value $C_W=2$ in Figure \ref{fig:init} precisely corresponds to the initialization \eqref{eq:He-gen-intro} since we've normalized the aggregation operator to ensure that $\lambda_1(\mP)=1$. Without this normalization, the homogeneity of ReLU networks means that omitting the $\lambda_1(\mP)^2$ term in the denominator of \eqref{eq:He-gen-intro} causes both activations and gradients to grow/decay like $\lambda_1(\mP)^{-\text{depth}}$. 

\begin{pt}\label{pt:he-init}
In practice, the spectrum of the aggregation operator depends on the input graph and hence varies between datapoints. However, Theorem \ref{thm:He-init-inf} shows that it is important to ensure that the top eigenvalue $\lambda_1(\mP)$ is (approximately) the same. Our practical recommendation is therefore to always normalize message passing operators to have a top eigenvalue equal to a fixed constant (say $1$) and then use He-initialization. This can be done by using normalized, rather than un-normalized, adjacency matrices, giving a potential explanation for why this is common in practice. 
\end{pt}

\begin{figure}
  \centering              
    \includegraphics[width=.7\linewidth]{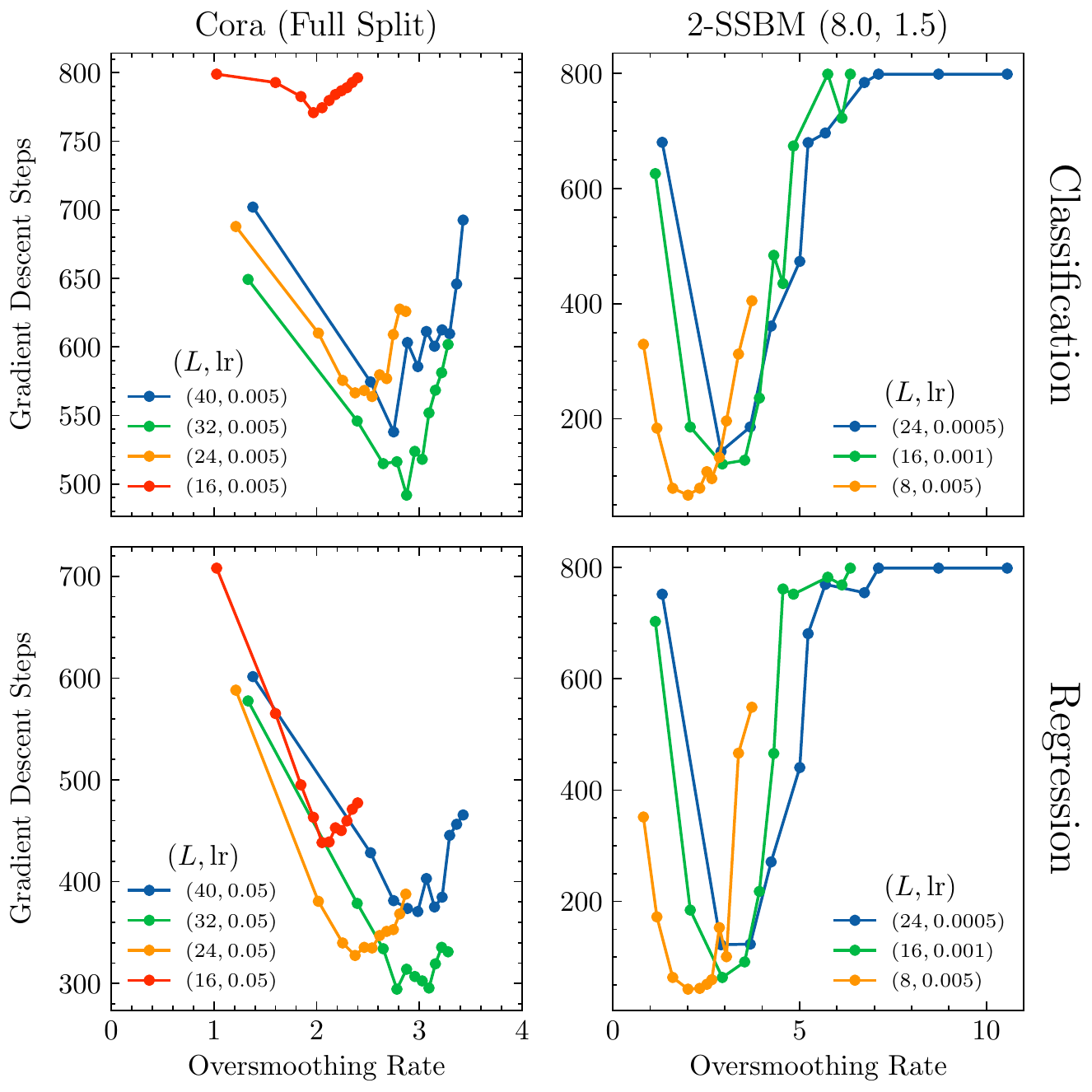} \\
  \caption{{\footnotesize Number of GD steps with constant step size $\mathrm{lr}$ required for depth $L$ GNNs with residual aggregation operators to reach the training accuracy of the best linear classifier, as a function of the oversmoothing rate at initialization. Models are trained on data from Cora and a two-class stochastic block model (2-SSBM), treated as either classification (cross-entropy loss) or regression (MSE with one-hot labels). Oversmoothing rates are varied by setting $C_W^{(\ell)}=2.25/(n_{\ell-1}\lambda_1(\mathcal{P})^2)$ and scanning over  values $t\in\{0.0, 0.1,\ldots,1.0\}$ for the residual aggregation strength. In all cases we use the symmetric degree-normalized adjacency matrix with self-loops $\mP$ as the base aggregation operator and take hidden layers widths to be $128$ for Cora and $64$ for SSBM. All models are allowed to train for a maximum of 800 steps. Oversmoothing rates for each depth and residual aggregation strength ($L$, $t$) are computed as an empirical average over all vertices in 25 random initializations. We conservatively report 800 steps for settings of ($L$, $t$) that never reached the training accuracy threshold. 
  }}
  \label{fig:os-bad}
\end{figure}

\subsection{Avoiding Oversmoothing at Init}\label{sec:os-inf} In a deep GNN where the same aggregation operator $\mP$ is applied at every layer it is natural to expect that for any network input $x$ and any neuron  $i=1,\ldots, n_L$ the resulting vector of pre-activations
\[
z_i^{{(L)}}(x) = \lr{z_{i;v}^{(L)}(x),\quad v\in V}\in \R^{|V|}
\]
in the final layer will be almost entirely contained in the top eigenspace of $\mP$. This effect is called oversmoothing and was first systematically pointed out in \cite{li2018deeper}. Building on the definition \eqref{eq:r-def} of the oversmoothing ratio, we measure the rate of oversmoothing in the network output by studying
\begin{equation}
    \text{oversmoothing rate}:=-\log\lr{1-r^{(L)}},
\label{eq:oversmoothing-rate}
\end{equation}
where we remind the reader that
\[
r^{(L)}=\E{\frac{1}{n_L}\sum_{i=1}^{n_L} \norm{\Pi_1 z_i^{(L)}(x)}^2} \bigg/ \E{\frac{1}{n_L}\sum_{i=1}^{n_L} \norm{z_i^{(L)}(x)}^2}
\]
is the oversmoothing ratio, which is independent of the network input $x$. Since $r^{(L)}\in [0,1]$, the oversmoothing rate diverges to $+\infty$ as more and more of the relative $\ell_2$-norm of the network outputs is contained in the top eigenspace of $\mP$.

As we illustrate in Figure \ref{fig:os-bad}, high oversmoothing rates at initialization lead to a significant slowdowns in network training (\S \ref{sec:lit-rev} for a discussion of empirical and theoretical related work on oversmoothing). Thus, it is practically important to identify strategies for mitigating oversmoothing, and we now present two theorems in this direction. The first is a negative result in Theorem \ref{thm:os-inf}, which shows that any initialization leads to oversmoothing at large depth in randomly initialized ReLU GNNs that do not have residual connections and use a fixed aggregation operator. The second, Theorem \ref{thm:res-os-inf}, is a positive result in which we prove that by simply using residual aggregation operators (see \eqref{eq:p-res-def}) oversmoothing rates can be effectively be controlled.

\begin{figure}
  \centering              
    \includegraphics[width=.7\linewidth]{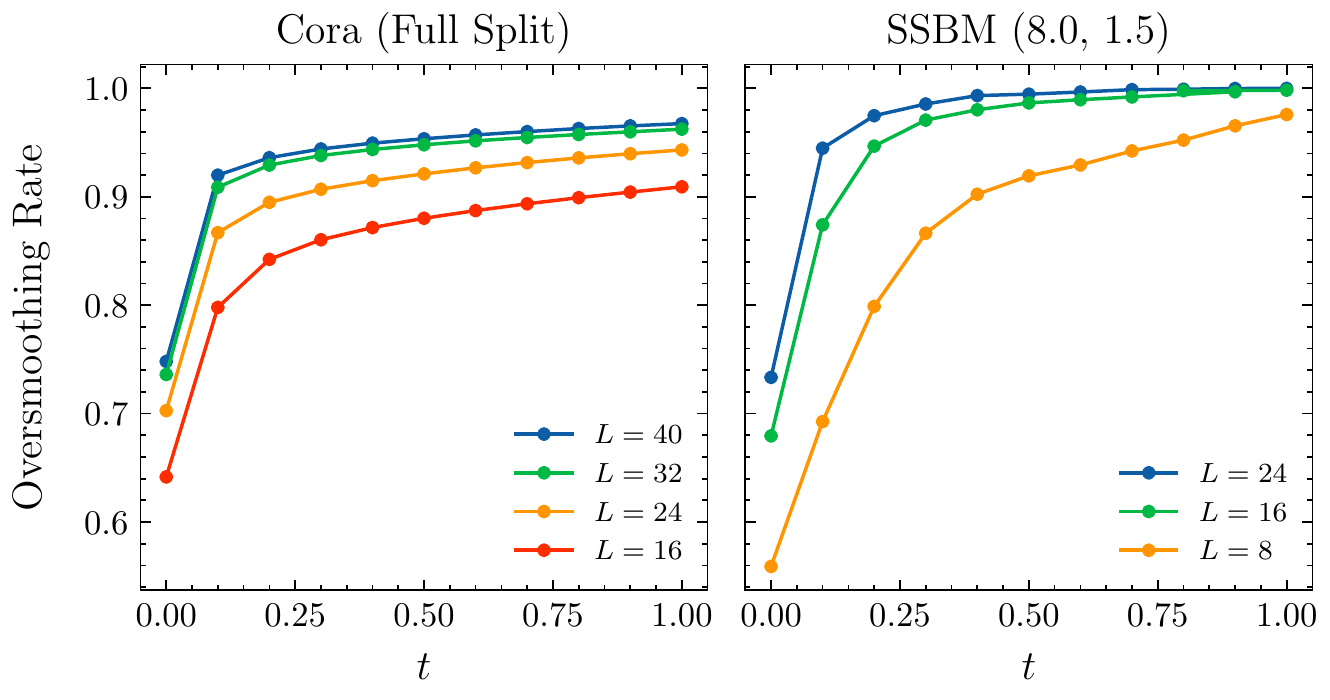} \\
  \caption{\footnotesize For each depth $L$ model reported in the classification portion of Figure \ref{fig:os-bad}, we plot the oversmoothing rate at initialization as a function of the residual aggregation strength $t$.} 
  \label{fig:os-fix}
\end{figure}

\begin{theorem}[Informal Consequence of Theorem \ref{thm:oversmoothing}]\label{thm:os-inf}
Consider a randomly initialized $L$-layer fully connected ReLU GNN with no residual connections and a fixed aggregation operator satisfying \eqref{eq:P-as1}-\eqref{eq:P-as4}. The oversmoothing ratio (see \eqref{eq:r-def}) converges to $1$ exponentially in $L$:
\begin{equation}\label{eq:r-est-inf-lb}
 r^{(L)} \geq 1-\exp\left[-\text{const}\times L\right].
\end{equation}
\end{theorem}
Lower bounds that are superficially similar  to  \eqref{eq:r-est-inf-lb} were obtained in prior work (see Corollary 2 in \cite{oono2019graph} and Theorem 3.4 in the followup \cite{cai2020note}). Specifically, these articles show that the distance between the layer $L$ pre-activations and their projection onto the top eigenspace of the aggregation operator satisfies the deterministic estimate
\begin{equation}\label{eq:prior-os}
\frac{1}{n_L}\sum_{i=1}^{n_L}\norm{z_i^{(L)}(x)-\Pi_1 z_i^{(L)}(x)}^2 \leq \text{const}\times \prod_{\ell=1}^L \frac{\norm{W^{(\ell)}}}{\lambda_1(\mathcal P)}
\end{equation}
where $\norm{W^{(\ell)}}$ is the spectral norm of $W^{(\ell)}$. For the randomly initialized weights we consider in the present article the result \eqref{eq:prior-os} is significantly weaker than our Theorem \ref{thm:He-init-inf}. To see this, recall from \S \ref{sec:evgp-inf} that to avoid exponential outputs we must initialize $W^{(\ell)}$ as in Theorem \ref{thm:He-init-inf}. Averaging \eqref{eq:prior-os} with respect to this distribution the left hand side remains bounded, independent of network depth. The normalized spectral norms $\norm{W^{(\ell)}}/\lambda_1(\mP)$, however, concentrate around the constant $2$ with high probability (this is a consequence of the $\sqrt{2}$ in the weight variance due to the ReLU and an extra $\sqrt{2}$ in the top singular value of Gaussian random matrices from the celebrated Bai-Yin Theorem \cite{bai2008limit} or Theorem 1.1 in \cite{bandeira2016sharp}). Hence, the upper bound in \eqref{eq:prior-os} is exponentially growing in the network depth (this is reflected for instance in Figure 2 of \cite{oono2019graph}) and quickly becomes vacuous.

While Theorem \ref{thm:os-inf} is formulated for GNNs with a fixed aggregation operator, the more general Theorem \ref{thm:oversmoothing} on which it is based also applies when in layer $\ell$ we use a layer-dependent \textit{residual aggregation operator} 
\begin{equation}\label{eq:p-res-def}
    \mP^{(\ell)}:=\lr{1-t_\ell}I + t_\ell \mP,\qquad t_\ell \in [0,1],
\end{equation}
where $\mP$ is a fixed aggregation operator satisfying \eqref{eq:P-as1}-\eqref{eq:P-as4}. A direct consequence of Theorem \ref{thm:oversmoothing} is that the rate of oversmoothing is exponential in the sum of the $t_\ell$'s:

\begin{theorem}[Informal Statement of Corollary \ref{cor:res-os} to Theorem \ref{thm:oversmoothing}]\label{thm:res-os-inf} Consider a randomly initialized $L$-layer fully connected ReLU GNN with residual aggregation operators as in \eqref{eq:p-res-def} built on a fixed aggregation operator $\mP$ satisfying \eqref{eq:P-as1}-\eqref{eq:P-as4}. Then
\[
1-\exp\left[-c_\mP \sum_{\ell=1}^{L-1} t_{\ell}\right] \leq r^{(L)}\leq \min \set{1,r^{(1)}\exp\left[C_\mP \sum_{\ell=1}^{L-1} t_{\ell}\right]},
\]
where $c_\mP, C_\mP$ are certain explicit positive constants depending on the spectrum of $\mP$. 
\end{theorem}
\begin{pt}\label{pt:res-ag}
In practice, using a fixed aggregation operator in deep fully connected ReLU GNNs is undesirable and results in many low frequency output features at initialization. This is illustrated in Figure \ref{fig:os-bad}. Instead, in an $L$ layer GNN to avoid oversmoothing  one may use residual aggregation operators of the form
\[
\mathcal P^{(\ell)}=\lr{1-\frac{\text{const}}{L}}I + \frac{\text{const}}{L} \mP,
\]
where $\mP$ is a fixed aggregation operator satisfying \eqref{eq:P-as1}-\eqref{eq:P-as4}. Residual aggregation operators are beneficial for node classification and node regression tasks (see Figures \ref{fig:cc-bad}, \ref{fig:cc-fix}).
\end{pt}

We wish to emphasize that residual aggregation operators and their variants are certainly not new and were proposed in combination with a variety of other interesting practical modifications of GNNs in prior work such as \cite{chen2020simple}. From this point of view, the main contribution of the present article is that we prove that, among other possible salutary effects,  residual aggregation operators  alleviate oversmoothing at a specific rate that depends on the choice of the $t_\ell$'s. We refer the reader to \S \ref{sec:lit-rev} for a more thorough discussion of related work.

\begin{figure}
  \centering              
    \includegraphics[width=.7\linewidth]{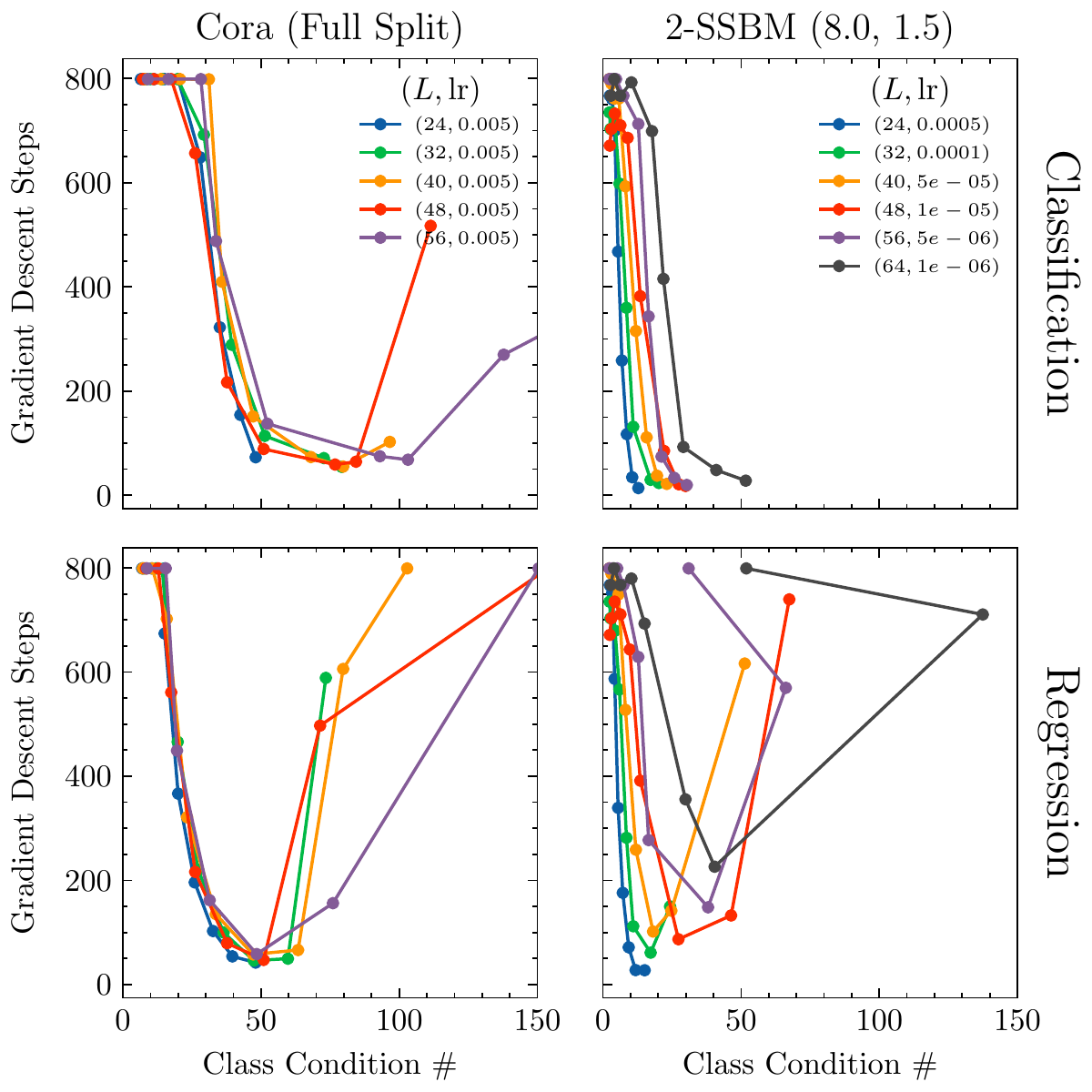} \\
  \caption{
  \footnotesize Number of GD steps with constant step size $\mathrm{lr}$ required for depth $L$ GNNs equipped with residual connections and residual aggregation operators to reach the training accuracy of the best linear classifier, as a function of the class condition number at initialization. Models are trained on data from Cora and a two class stochastic block model (SSBM) graph, treated as either classification (cross-entropy loss) or regression (MSE with one-hot labels). Class condition numbers are varied by setting $C_W^{(\ell)}=1/(n_{\ell-1}\lambda_1(\mathcal{P})^2)$ and scanning over values $\beta\in\{0.1,0.2,\ldots,0.9\}$ for the residual connection strength. In all cases we use the symmetric degree-normalized adjacency matrix with self-loops $\mP$ as the base aggregation operator and take hidden layers widths to be $128$ for Cora and $64$ for SSBM. To mitigate oversmoothing, at each depth $L$ and residual connection strength $\beta$, we scan over models with residual aggregation operators $(1-t)I + t \mP$ with $t\in \{0.1, 0.2, 0.3, 0.4, 0.6, 0.8\}$, reporting results only for the model that trained in the fewest number of GD steps. All models are allowed to train for a maximum of 800 steps Class condition numbers for each setting of depth, residual connection strength, and residual aggregation strength ($L$, $\beta$, $t$) are computed as an empirical average over all vertices in 25 random initializations. For each initialization and setting of ($L$, $\beta$, $t$) that never reached the training accuracy threshold we conservatively report 800 steps.
  }
  \label{fig:cc-bad}
\end{figure}

\begin{figure}
  \centering              
    \includegraphics[width=.8\linewidth]{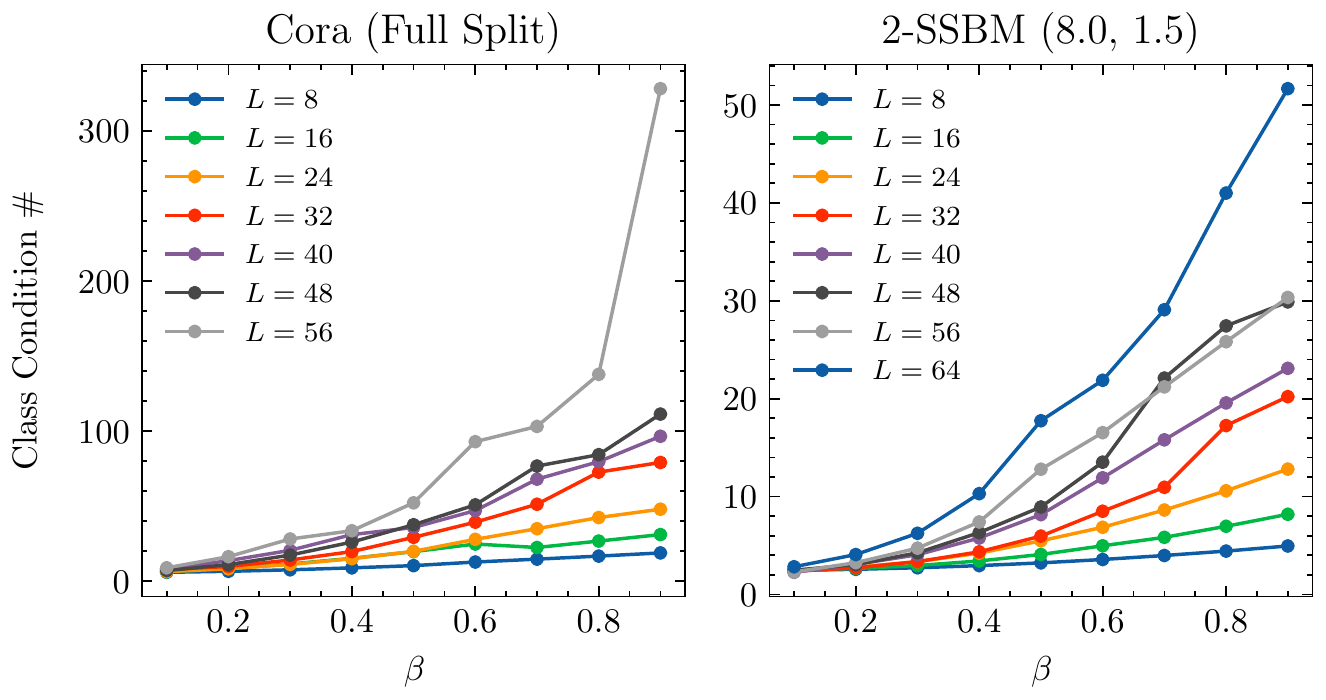} \\
  \caption{\footnotesize For each depth $L$ model reported in the classification portion of Figure \ref{fig:cc-bad}, we plot the class condition number at initialization as a function of the residual connection strength $\beta$.}
  \label{fig:cc-fix}
\end{figure}

\subsection{Avoiding Correlation Collapse at Init}\label{sec:ntk-inf}
Our final set of results are primarily empirical and give a simple procedure for combining initialization scheme and network architecture to ensure that hidden layer features remain well-conditioned at the start of training. To make this precise, let us fix a graph $\mG = (V,E)$ and an $L$-layer ReLU GNN. For a $k$ class classification (or one-hot regression) problem in which each vertex $v\in V$ is assigned a class $c(v)\in \set{1,\ldots, k}$ we fix a network input $x$ and consider the class-averaged features in the final hidden layer:
\[
z_c^{(L)}(x):=\frac{1}{\#\set{v\in V~|~c(v)=c}}\sum_{\substack{v\in V\\ c(v)=c}} z_{v}^{(L)}(x).
\]
We then measure correlation collapse by studying 
\begin{equation}
    \text{class condition number}:=\text{ condition number of }\lr{z_c^{(L)}(x),\,\, c=1,\ldots, k}\in \R^{k\times n_L},
\label{eq:class-condition}
\end{equation}
which becomes large when output features are highly correlated crosses different classes. As we illustrate empirically in Figure \ref{fig:cc-bad}, very small and very large values for the class condition number noticeably slow early training dynamics. 

We that a practical mechanism for overcoming correlation collapse, are skip connections such that the forward pass evaluated at a fixed network input $x$ is given by
\begin{equation}\label{eq:z-res-def}
z^{(\ell+1)}(x) =
\begin{cases}
 z^{(\ell)}(x)+\beta_\ell \mP^{(\ell)}\sigma\lr{z^{(\ell)}(x)}W^{(\ell+1)},&\quad \ell\geq 1\\
  x+\beta_1 \mP^{(1)}x W^{(1)},&\quad \ell=0
\end{cases}
\end{equation}
Here $\beta_\ell \geq 0$ is the strength of the residual connection at layer $\ell$ and $\mP^{(\ell)}$ are either fixed aggregation operators satisfying \eqref{eq:P-as1}-\eqref{eq:P-as4} or residual aggregation operators as in \eqref{eq:p-res-def}.

Our key insight is that smaller values for $\beta_\ell$ cause network outputs to be closer to network inputs and hence mitigate correlation collapse in practice. This is illustrated in Figure \ref{fig:cc-fix}. The reason this occurs is that the recursion \eqref{eq:z-res-def} immediately yields
\[
\norm{z^{(\ell+1)}(x)-z^{(\ell)}(x)}^2 \leq C \beta_\ell^2,
\]
in which the constant $C$ depends only on $\mathcal P^{(\ell)}$ and $\norm{z^{(\ell)}(x)}^2$. Iterating this inequality shows that if $\beta_\ell$ scale like $\text{const}/\sqrt{L}$ for a sufficiently small constant, then network inputs and outputs are very similar. 


\begin{pt}\label{pt:res-2}
In practice, using residual connections \eqref{eq:z-res-def} is important for encouraging trainability in deep GNNs (see Figures \ref{fig:ablation-performance}, \ref{fig:ablation-time}). That residual connection alleviate correlation collapse is illustrated in Figure \ref{fig:cc-fix}.
\end{pt}

\section{Review of Literature}\label{sec:lit-rev}
Before presenting our theoretical results in \S \ref{sec:results}, we place our work into context by reviewing several strands of related literature. The first is work on the so-called information propagation approach to choosing good network initializations in which the goal is essentially to avoid the exponential outputs and correlation collapse  failures modes from the Introduction, which are not specific to graph neural networks. This line of work started with the articles \cite{poole2016exponential,schoenholz2016deep}, which propose a principled way to select weight and bias variances so that the forward and backward pass at infinite width are well-behaved at initialization. These ideas were then extended to fully connected and residual networks with ReLU activations at finite width in \cite{hanin2018neural,hanin2018start}. For general activations at both infinite and finite width a more refined analysis of fully connected networks at the start of training was carried out in \cite{hanin2022random,roberts2022principles}. The case of transformers was recently taken up in \cite{dinan2023effective}. Unlike these articles, the present work applies to graph neural networks. In this context, equation (14) in the article \cite{jaiswal2022old} proposes an initialization scheme in which weight variances are re-scaled by the squared Hilbert-Schmidt norm of the aggregation operators. We find instead that it is better to use the spectral norm (see Theorems \ref{thm:He-init-inf} and \ref{thm:He-init} as well as Figure \ref{fig:evgp-bad}). 

The next line of prior work we briefly review explores and considers various solutions to the oversmoothing problem, which was originally identified in \cite{li2018deeper}. Articles such as, for instance \cite{chen2020measuring,yang2020revisiting}, propose regularizers that penalize oversmoothing. Moreover, works such as \cite{yan2022two,liu2020towards,lu2021skipnode,zheng2021deeper} introduce various reweighting and normalization schemes, while articles such as \cite{li2021training} explore using reversible connections. In a somewhat different vein, the article \cite{godwin2021simple} proposes a ``noisy nodes'' paradigm in which input vertex features (and sometimes edges and targets as well) are corrupted by noise and an auxiliary denoising autoencoder loss is added to present correlation collapse in the final hidden layer. A core finding of \cite{godwin2021simple} is that such regularization schemes allow for training of deeper GNNs and help alleviate oversmoothing. Finally, articles such as \cite{wu2022non} provide a non-asymptotic analysis for rates of oversmoothing and distinguish between its relation to denoising and debiasing. Prior work has also shown that oversmoothing occurs in certain \textit{linear} GNNs at finite width \cite{yang2020revisiting,keriven2022not} and in non-linear networks at infinite width \cite{sabanayagam2021new,huang2021towards} by studying the properties of the neural tangent kernel in GNNs.

The preceding articles on oversmoothing often propose using various kinds of skip connections in GNNs. We prove here that while skip connections are not necessary for avoiding oversmoothing (see Theorems \ref{thm:res-os-inf} and \ref{thm:oversmoothing} as well as Figure \ref{fig:os-fix}) they are indeed useful in avoiding correlation collapse (see \S \ref{sec:ntk-inf} and Figures \ref{fig:cc-bad}, \ref{fig:cc-fix}). Residual connections were already proposed and shown to have some empirical benefits, of instance, in the article \cite{kipf2016semi} that originally introduced graph neural networks (see (14) in Appendix B) and have been studied empirically in \cite{chen2020simple,li2021deepgcns}. The effect of residual connections on the graph NTK (i.e. the covariance matrix of gradients of the model output with respect to the parameters) was then analyzed more precisely in articles such as \cite{huang2021towards} which considered residual connections of the form 
\[
z^{(\ell+1)}(x) = \alpha z^{(\ell)}(x)+ \beta \mP^{(\ell)}\sigma\lr{z^{(\ell)}(x)}W^{(\ell+1)},
\]
where $\alpha,\beta\in [0,1]$. Theorems 2 and 3 in \cite{huang2021towards} show that if $\beta=\lr{1-\alpha^2}^{1/2} < 1$ then the NTK has is close to rank $1$ at large depth, leading to both oversmoothing and correlation collapse. However, in this article, we see that if $\alpha=1$ and $\beta$ scales like a constant times $L^{-1/2}$, then both these failure modes can be avoided. This is essentially the analog of the FixUp initialization introduced for general ResNets in \cite{zhang2019fixup}.

Another kind of residual connection proposed in the literature inserts a skip connection not between each consecutive pair of layers but rather connects the pre-activations in each layer directly to the network input. The article \cite{yang2020revisiting}, for instance, investigates empirically the use of such skip connections for signal denoising. In contrast, the work \cite{sabanayagam2021new} gives an graph-NTK point of view on the salutary effects of this sort of skip connection in avoiding correlation collapse for model gradients. 

Finally, we single out \cite{chen2020simple} as being particularly related to the present work. This article, though completely empirical, introduced a residual GNN architecture called GCNII, which simultaneously incorporates both the FixUp-type skip connections and the residual aggregation operators we study here. The resulting models were shown empirically to achieve strong results on Cora and related benchmark tasks. 

\section{Theoretical Results}\label{sec:results}
In this section, we state our main theoretical results. Specifically, in \S \ref{sec:evgp-res} we give a formal version of Theorem \ref{thm:He-init-inf} regarding the exponential output failure mode in GNNs. We also explain explain its precise connection to the exploding and vanishing gradient problem in \S \ref{sec:evgp}. Then, in \S \ref{sec:os-res}, we state a precise version of the oversmoothing estimates in Theorems \ref{thm:os-inf} and \ref{thm:res-os-inf}. Our results are all phrased in terms of the $\abs{V}\times \abs{V}$ matrix of layer $\ell$ pre-activation covariances 
\begin{align}
    \label{eq:K-def} K^{(\ell)}(x)&=\E{z_i^{(\ell)}(x)\lr{z_i^{(\ell)}(x)}^T}
\end{align}
and the analogous $\abs{V}n_0 \times \abs{V}n_0$ matrix of layer $\ell$ gradient covariances 
\begin{equation}\label{eq:G-def}
    G^{(\ell)}(x)=\E{D_x z_i^{(\ell)}(x)\lr{D_x z_i^{(\ell)}(x)}^T}
\end{equation}
in which we've denoted by $D_x$ the Jacobian with respect to $x$. Note that both $G^{(\ell)}(x)$ and $K^{(\ell)}(x,x')$ are independent of the neural index $i$ by symmetry. 

\subsection{Avoiding Exponential Outputs by Generalizing He-Initialization to GNNs}\label{sec:evgp-res}
Our first result gives recursive estimates in $\ell$ for the output and gradient covariances in a randomly initialized ReLU GNN at initialization. 
\begin{theorem}\label{thm:He-init}
Consider an $L$ hidden layer ReLU GNN at initialization as in \S \ref{sec:def}. For each $\ell=1,\ldots, L$, we have
\begin{equation}\label{eq:M-est}
A_\ell \tr(K^{(1)}(x)\Pi_1)\leq \tr(K^{(\ell+1)}(x))\leq A_\ell\tr(K^{(1)}(x)),    
\end{equation}
where 
\[
A_\ell:=\prod_{\ell'=1}^\ell a_{\ell'},\quad a_\ell =  \frac{1}{2}C_W^{(\ell+1)}n_{\ell}\lambda_1(\mP^{(\ell)})^2,
\]
and by convention $\lambda_1(\mP^{(0)})=1$. Moreover, the estimate \eqref{eq:M-est} holds also with $K^{(\ell)}(x)$ replaced by $G^{(\ell)}(x).$
\end{theorem}

\noindent We prove Theorem \ref{thm:He-init} in \S \ref{sec:He-init-pf}. The relation \eqref{eq:M-est} immediately gives the following 
\begin{corollary}\label{cor:He-init}[Generalization of He-Initialization]
Setting the weight variance according to
\[
C_W^{(\ell+1)} =  \frac{2}{n_\ell\lambda_1(\mP^{(\ell)})^2 }
\]
ensures that the traces of the second moment matrices $K^{(L)}(x,x)$ of output activations and $G^{(L)}(x,x)$ of input-output gradients are bounded from above and below by positive constants uniformly in $L$.

\end{corollary}
\subsubsection{Connecting Theorem \ref{thm:He-init} to the Exploding and Vanishing Gradient Problem}\label{sec:evgp} 
Before presenting our results on oversmoothing, let give a non-rigorous but intuitive explanation of why ensuring that $K^{(\ell)}$ and $G^{(\ell)}$ are neither growing nor decaying in $\ell$ is related to avoiding the exploding and vanishing gradient problem (EVGP) \cite{SchmidtHuber,bengio1994learning}. To keep the notation as simple as possible, let us consider the case of models with output dimension $1$. Recall that the EVGP concerns the size of the derivatives of the loss $\ell= \ell(z^{(L+1)}(x;\theta),y)$ on a training datapoint $(x,y)$ with respect to various network weights $W_{ij}^{(\ell)}$. Let us write 
\[
\partial_{W_{ij}^{(\ell)}} \ell = \partial_{W_{ij}^{(\ell)}} z^{(L+1)} \partial_{z^{(L+1)}}\ell.
\]
The second term involves only the values of $z^{(L+1)}$ and the derivatives of $\ell(z,y)$ with respect to $z$. Heuristically, these are order $1$ as soon as $K^{(L)}$ is order $1$. Next, by the chain rule,
\begin{align*}
\partial_{W_{ij}^{(\ell)}} z^{(L+1)} &= \partial_{W_{ij}^{(\ell)}} z_{i}^{(\ell)}\partial_{z_i^{(\ell)}} z^{(L+1)}= \sigma(z_j^{(\ell-1)})\partial_{z_i^{(\ell)}} z^{(L+1)}.
\end{align*}
Reasoning as above, we see that the first term is order $1$ as long as $K^{(\ell-1)}$ has order $1$. Finally, the second term has the form of an input-output Jacobian where we treat the pre-activations at layer $\ell$ as the input to a ReLU GNN with depth $L-\ell$. Hence, if such input-output Jacobians have order $1$, then parameter gradients also have order $1$. Making the preceding precise requires one to compute also higher order moments of network activations and input-output Jacobian, which we do not take up in this article.

\subsection{Avoiding Oversmoothing in ReLU GNNs Through Residual Aggregation Operators}\label{sec:os-res}
In this section, we state our main results on oversmoothing. For this, let us recall that we denoted in \eqref{eq:r-def} by $r^{(\ell)}$ the oversmoothing ratio after $\ell$ layers.

\begin{theorem}\label{thm:oversmoothing}
Consider an $L$ hidden layer ReLU GNN at initialization as in \S \ref{sec:def}.
For any positive values of the weight variances $C_W^{(\ell)}$ we have 
\begin{align}
\label{eq:os-rate-lb} &r^{(\ell+1)}\geq \lr{1+\lr{\frac{1}{r^{(1)}}-1} \prod_{\ell'=1}^{\ell} \lr{\frac{\lambda_2(\mP^{(\ell')})}{\lambda_1(\mP^{(\ell')})}}^2}^{-1} 
\end{align}
and
\begin{equation}\label{eq:os-rate-ub}
    r^{(\ell+1)} \leq \min \set{1,\,\lr{\prod_{\ell'=1}^\ell \frac{\lambda_1(\mP^{(\ell)})}{\lambda_{\abs{V}}(\mP^{(\ell)})}}^2 r^{(1)}}.
\end{equation}
Thus, if for all $\ell$ we have $\mP^{(\ell)}=\mP$ with $\mP$ not a multiple of the identity, then the oversmoothing ratio converges to $1$ exponentially in the network depth. 
\end{theorem}
\noindent We prove Theorem \ref{thm:oversmoothing} in \S \ref{sec:oversmoothing-pf}. The estimates \eqref{eq:os-lb} and \eqref{eq:os-ub} immediately give the following
\begin{corollary}[Residual Aggregation Operators Avoid Oversmoothing]\label{cor:res-os}
Define
\[
\mP^{(\ell)}:=\lr{1-t_\ell}I+t_\ell \mP,\qquad t_\ell \in [0,1]
\]
for a fixed symmetric, positive definite aggregation operator $\mP$  satisfying \eqref{eq:P-as1}-\eqref{eq:P-as4}. Then the oversmoothing ratio $r^{(L+1)}$ of the network outputs is bounded below by
\begin{equation}\label{eq:os-lb}
\lr{1+\exp\left[\frac{\lambda_2(\mP)-\lambda_1(\mP)}{\max\set{1,\lambda_1(\mP)}}\sum_{\ell=1}^L t_{\ell}\right]}^{-1}
\end{equation}
and above by the minimum of $1$ and
\begin{equation}\label{eq:os-ub}
  \exp\left[\frac{\lambda_{1}(\mP)-\lambda_{\abs{V}}(\mP)}{\min\set{1,\lambda_{\abs{V}}(\mP)}}\sum_{\ell=1}^L t_{\ell}\right] \lr{1-\frac{1}{r^{(1)}}}^{-1}.
\end{equation}
Hence, the rate of oversmoothing is bounded away from $1$ as long as $\sum_{\ell=1}^L t_{\ell}$ is sufficiently small.
\end{corollary}
We prove Corollary \ref{cor:res-os} in \S \ref{sec:res-os-pf}.

\section{Experiments}\label{sec:exp}
\subsection{Semi-Supervised Node Classification Datasets}
\label{sec:experiments-datasets}
The experiments shown in Figs.~\ref{fig:init}-\ref{fig:cc-fix} are conducted using citation networks and stochastic block models (SBMs). Citation networks like Cora, Citeseer, and PubMed (see Table~\ref{tab:citation-networks}) are graphs with vertices representing publications and edges representing citations. Vertex features are binary, indicating the presence or absence of a word in each publication. Each vertex is given a class label representing the field in which the corresponding publication belongs. Vertices are partitioned into train, test, and validation sets. During training, a model generates vertex-level predictions for full graph, but only vertices in the train set are used to update the model. For this reason, the learning task is called semi-supervised classification. 

SBMs~\cite{abbe2018community} are random graph models used to study community detection, the task of clustering vertices belonging to the same community. SBMs are specified by by the types of graphs they generate. In symmetric SBMs (SSBMs), graphs are generated with $n$ vertices distributed with equal probability between $k$ classes. Edges are randomly drawn between in-class vertices with probability $a\log n/ n$ and out-of-class vertices with probability $b\log n/n$. The relative values of $a$ and $b$ modulate the difficulty of the community detection problem; when $|\sqrt{a}-\sqrt{b}|>\sqrt{2}$, communities are exactly recoverable~\cite{abbe2016exact}. In this work we focus on $k=2$ SSBMs and commonly label the corresponding graphs as 2-SSBM($a$, $b$). We train GNNs to perform community detection via semi-supervised classification by randomly partitioning vertices into train (50\%), test (25\%), and validation (25\%) sets. SBMs in this paper are generated via \textsc{PyTorch Geometric}~\cite{fey2019fast} and endowed with randomly sampled vertex features roughly correlated with class label (see the \textsc{scikit learn} \texttt{make\_classification} function). 

\begin{table}[htb!]
  \centering
  \label{tab:summary}
  \begin{tabular}{lllllll}
    \hline
    Dataset & Vertices & Edges & Features & Classes & $\lambda_1(P)$ & Degeneracy  \\
    \hline
    Cora & 2708 & 10556 & 1433 & 7 & 1.0000 & 78\\
    Citeseer & 3327 & 9104 & 3703 & 6 & 1.0000 & 438\\
    PubMed & 19717 & 88648 & 500 & 3 & 1.0000 & 1 \\
    \hline
  \end{tabular}
  \caption{Citation networks used for semi-supervised classification. We use the versions available in \textsc{PyTorch Geometric} and train, test, and validation masks generated using the so-called ``full split.''}
  \label{tab:citation-networks}
\end{table}

\begin{table}[htb!]
  \centering
  \label{tab:summary-2}
  \begin{tabular}{lllllll}
    \hline
     a & b & Vertices & Edges & $\lambda_1(P)$ & $\lambda_2(P)$ & Recovery \\
    \hline
    8 & 1.5 & 800 & 25784 & 1.000 & 0.710 & Exact\\
    4 & 3 & 800 & 19204 & 1.000 & 0.412 & None\\
    \hline
  \end{tabular}
  \caption{2-SSBMs used for semi-supervised node classification. }
\end{table}

\subsection{Performance Scans}
\label{sec:experiments-performance}
In semi-supervised node classification tasks, one GNN forward pass returns node-level predictions for the full graph. During training, the loss on the train set of vertices is backpropagated and used to update the model's weights. Predictions on the validation and test vertices are used to evaluate the performance of the model at the corresponding train step. When studying a model's final performance, as in Figs.~\ref{fig:ablation-performance} and \ref{fig:ablation-time}, we report its performance on the test vertex set at the train step that maximizes performance on the validation vertex set. To suppress fluctuations and favor stable convergence, the validation performance of a model is only considered to be maximized if either 1) the validation accuracy has increased for the last 25 epochs or 2) the validation accuracy has been stable within 2.5\% of its current value for the last 25 epochs. We report the test set accuracy of the model at the train step satisfying these conditions and minimizing the validation loss. 

\subsection{Time-to-Train Scans}
\label{sec:experiments-time}

In Figs.~\ref{fig:evgp-bad}, \ref{fig:os-bad}, and \ref{fig:cc-bad} we report the number of gradient descent steps with a constant learning rate required for GNNs to reach the training accuracies of the best linear classifier. We quantify this performance by training a single linear layer to produce classification scores for each vertex in the graph. We scan over $\mathrm{lr}\in\{0.05, 0.01, 0.005, 0.001, 0.0005\}$ and report the highest average validation accuracy (noting that linear classifiers perform equally well on the validation and test sets). This procedure yields 68\% for Cora, 87\% for the 2-SSBM(8,1.5), and 74\% for the 2-SSBM(4,3). When scanning for the fastest training times, our goal is to identify stable convergence above these training accuracy thresholds. To suppress fluctuations, we require that the model remain above the threshold for 10 epochs and either 1) the validation accuracy has increased for the last 10 epochs or 2) the validation accuracy has been stable within 5\% of its current value across the last 10 epochs.

\section{Proofs} \label{sec:proofs}
This section contains detailed proofs of the Theorems presented in \S \ref{sec:results}. We begin in \S \ref{sec:prep} with some preparatory Lemmas. 

\subsection{Preparatory Results}\label{sec:prep}
In this section, we record several results that will be used in the course of proving Theorems \ref{thm:He-init} and \ref{thm:oversmoothing}.
The first result is well-known (see e.g. Proposition 2 \cite{hanin2020products}). 
\begin{lemma}\label{lem:half}
Consider a depth $L$ ReLU GNN with hidden layers of widths $n_\ell$, aggregation operators $\mP^{(\ell)}$ satisfying \eqref{eq:P-as1}-\eqref{eq:P-as4} and random weights as in \eqref{eq:rand-def}. Fix any $\ell=1,\ldots, L+1$, any $v\in V$, and any network input $x=\lr{x_{v;i},\, v\in V,\, 1\leq i\leq n_0}.$ Then, 
\begin{equation}\label{eq:half-val}
\E{\lr{\sigma(z_{v;i}^{(\ell)}(x))}^2} = \frac{1}{2}\E{\lr{z_{v;i}^{(\ell)}(x)}^2}.    
\end{equation}
Similarly, for any $v_0\in V,1\leq i_0\leq n_0$ and every $\ell\geq 2$ we find
\begin{equation}\label{eq:half-deriv}
\E{\lr{\partial_{x_{v_0;i_0}}\sigma(z_{v;i}^{(\ell)}(x))}^2} = \frac{1}{2}\E{\lr{\partial_{x_{v_0;i_0}}z_{v;i}^{(\ell)}(x)}^2}.    
\end{equation}
\end{lemma}
\begin{proof}
Note that if $X\sim \mN(0,\rho^2)$ is any centered Gaussian, then $\abs{X}$ is independent of $\mathrm{sgn}(X)$. Moreover, when $\rho>0$ we also have that $\abs{X}$ is also independent of
\[
{\bf 1}_{X\geq 0} = 2\mathrm{sgn}(X)-1,
\]
which is equal in distribution to a $\text{Bernoulli}(1/2)$ random variable. To apply this, consider the event
\[
S^{(\ell-1)}:=\set{z^{(\ell-1)}(x)\neq 0}
\]
and note that, conditional on $z^{(\ell-1)}(x)$, the pre-activation $z_{v;i}^{(\ell)}$ is a centered Gaussian whose variance is non-zero on the event $S^{(\ell-1)}$. Hence, we have 
\begin{align*}
\E{\lr{\sigma(z_{v;i}^{(\ell)}(x))}^2} &=\E{\lr{\sigma(z_{v;i}^{(\ell)}(x))}^2 {\bf 1}_{ S^{(\ell-1)}}}\\
&= \E{\E{\lr{\sigma(z_{v;i}^{(\ell)}(x))}^2~\big|~z^{(\ell-1)}(x)}{\bf 1}_{ S^{(\ell-1)}}}\\
&=\E{\E{{\bf 1}_{z_{v;i}^{(\ell)}(x)\geq 0} \lr{z_{v;i}^{(\ell)}(x)}^2~\big|~z^{(\ell-1)}(x)}{\bf 1}_{ S^{(\ell-1)}}}\\
&=\frac{1}{2}\E{\E{  \lr{z_{v;i}^{(\ell)}(x)}^2~\big|~z^{(\ell-1)}(x)}{\bf 1}_{ S^{(\ell-1)}}}\\
&=\frac{1}{2}\E{\lr{z_{v;i}^{(\ell)}(x)}^2{\bf 1}_{ S^{(\ell-1)}}}\\
&=\frac{1}{2}\E{\lr{z_{v;i}^{(\ell)}(x)}^2}
\end{align*}
proving \eqref{eq:half-val}. The derivation of \eqref{eq:half-deriv} is similar.
\end{proof}

The next result we need is the following

\begin{lemma}\label{lem:sigma-trick}
Suppose that $z_{1},z_{2}$ are centered and jointly Gaussian random variables with covariance 
\[
K=\twomat{K_{11}}{K_{12}}{K_{21}}{K_{22}}.
\]
Then, with $\sigma(t)=\mathrm{ReLU}(t)=t{\bf 1}_{\set{t\geq 0}}$, we have
\begin{equation}\label{eq:sigma-trick}
\E{\sigma(z_{1})\sigma(z_{2})} \geq \frac{1}{2}K_{12}.    
\end{equation}
\end{lemma}
\begin{proof}
We have
\[
\E{\sigma(z_1)\sigma(z_2)}= \E{z_1z_2{\bf 1}_{\set{z_1\geq 0,\, z_2\geq 0}}}.
\]
By averaging over the change of variables $(z_1,z_2)\mapsto (-z_1,z_2)$ we find
\[
\E{\sigma_{v_1}\sigma_{v_2}}= \frac{1}{2}\E{z_1z_2{\bf 1}_{\set{z_1z_2\geq 0}}}.
\]
Hence, 
\begin{align*}
  2\E{\sigma_{v_1}\sigma_{v_2}}- K_{v_1v_2} =  \E{z_1z_2{\bf 1}_{\set{z_1z_2\geq 0}}} - \E{z_1z_2}  = \E{\abs{z_1z_2}{\bf 1}_{\set{z_1z_2\leq 0}}}\geq 0.
\end{align*}
\end{proof}

We will also need

\begin{lemma}\label{lem:pi-pos}
Suppose $V$ is a finite set and $\mP\in \R^{\abs{V}\times\abs{V}}$ is a symmetric matrix with non-negative entries. Then there exists an orthonormal basis 
\[
\pi_j=\lr{(\pi_j)_v,\, v\in V},\quad j=1,\ldots, d_1(\mP) = \mathrm{dim}\lr{\ker\lr{\mP - \lambda_1(\mP)I}}
\]
of the top eigenspace of $\mP$ consisting of elements with non-negative components:
\[
(\pi_j)_v \geq 0\quad \forall \, j=1,\ldots, d_1(\mP),\, v\in V.
\]
\end{lemma}
\begin{proof}
This is essentially a variant of the Perron-Frobenious theorem for irreducible non-negative matrices. To see this, consider the following equivalence relation on $V$:
\[
v_1\sim v_2\qquad \Longleftrightarrow \qquad \exists u_1,\ldots, u_k\in V\text{ s.t. } u_1 = v_1,\, u_k=v_2,\, \mP_{u_1u_2},\ldots, \mP_{u_{k-1}u_k}>0.
\]
Denoting by $c$ the number of equivalence classes and relabeling the elements of $V$ so that elements of the same equivalence class are numbered consecutively, we may write
\[
\mP = \lr{\begin{array}{cccc}
    \mP_1 & 0& \cdots & 0  \\
    0 & \mP_2 & \ddots&\vdots\\
    \vdots& \ddots&\ddots& 0\\
    0&\cdots & 0 & \mP_c
\end{array}}.
\]
The key property of each $\mP_j$ is that the graphs 
\[
\mG_{\mP_j}:=(V,E_{\mP_j}),\qquad (v_1,v_2)\in E_\mP ~\Leftrightarrow \mP_{v_1v_2}>0
\]
are connected. Thus, each $\mP_j$ is an irreducible symmetric non-negative matrix. By the Perron-Frobenious theorem we then have that:
\begin{itemize}
    \item the spectral radius of $\mP_j$ is equal to its largest eigenvalue $\lambda_1(\mP_j)$
    \item the eigenspace of $\mP_j$ with eigenvalue $\lambda_1(\mP_j)$ has dimension $1$ and is spanned by an eigenvector with non-negative entries 
\end{itemize}
Combining these properties with the block-diagonal decomposition of $\mP$ above completes the proof.
\end{proof}

The final preparatory result we will need is the following

\begin{lemma}\label{L:interference}
Suppose $G = (E,V)$ is a graph and $\mP\in \R^{\abs{V}\times\abs{V}}$ is a symmetric matrix with the following spectral decomposition:
\[
\mP = \sum_{j=1}^{\abs{V}} \lambda_j \pi_j\pi_j^T,\qquad \pi_i\cdot \pi_j = \delta_{ij},\qquad \pi_i\in \R^{\abs{V}}
\]
with $\pi_1=\lr{(\pi_1)_v,\, v\in V}$ satisfying
\[
\inf_{v\in V} (\pi_1)_v^2 =:\delta\geq 0.
\]
Then, 
\[
\sup_{v_1\neq v_2\in V}\lr{\sum_{j=2}^{\abs{V}}\pi_j\pi_j^T}_{v_1v_2} = \sup_{v_1\neq v_2\in V} \sum_{j=2}^{\abs{V}}\lr{\pi_j}_{v_1}\lr{\pi_j}_{v_2} \leq  -\delta.
\]
\end{lemma}
\begin{proof}
We have
\[
\lr{\sum_{j=1}^{\abs{V}} \pi_j\pi_j^T}_{v_1v_2} =\delta_{v_1v_2}.
\]
Hence, for $v_1\neq v_2$ we obtain 
\[
\lr{\sum_{j=2}^{\abs{V}} \pi_j\pi_j^T}_{v_1v_2} = - (\pi_{1})_{v_1}(\pi_{1})_{v_2}\leq-\delta
\]
completing the proof.
\end{proof}

\subsection{Proof of Theorem \ref{thm:He-init}}\label{sec:He-init-pf}
The upper bounds in \eqref{eq:M-est} follow directly by iterating the upper bounds \eqref{eq:tr-bounds-K} and \eqref{eq:tr-bounds-G} in the following result

\begin{proposition}\label{prop:K-rec}
For each $\ell=1,\ldots, L$ we have
\begin{equation}\label{eq:K-rec}
K^{(\ell+1)} = C_W^{(\ell+1)}n_{\ell} \mP^{(\ell)} \E{\sigma\lr{z_i^{(\ell)}}\sigma\lr{z_i^{(\ell)}}^T} \mP^{(\ell)}.    
\end{equation}
Similarly, for any $v_{0},v_0'\in V,\, 1\leq i_0,i_0'\leq n_0$ we find
\begin{equation}\label{eq:G-rec}
\E{\lr{\partial_{x_{v_0;i_0}}z_{i}^{(\ell+1)}}\lr{\partial_{x_{v_0';i_0'}}z_{i}^{(\ell+1)}}} = C_W^{(\ell+1)}n_{\ell} \mP^{(\ell)} \E{\partial_{x_{v_0;i_0}}\sigma\lr{z_i^{(\ell)}}\lr{\partial_{x_{v_0';i_0'}}\sigma\lr{z_i^{(\ell)}}}^T} \mP^{(\ell)}.    
\end{equation}
In particular, we have
\begin{equation}\label{eq:tr-bounds-K}
\frac{1}{2}C_W^{(\ell+1)}n_{\ell}\lambda_{\text{min}}(\mP^{(\ell)})^2 \leq \frac{\tr\lr{K^{(\ell+1)}}}{\tr(K^{(\ell)})}\leq \frac{1}{2}C_W^{(\ell+1)}n_{\ell}\lambda_{\text{max}}(\mP^{(\ell)})^2  
\end{equation}
and
\begin{equation}\label{eq:tr-bounds-G}
\frac{1}{2}C_W^{(\ell+1)}n_{\ell}\lambda_{\text{min}}(\mP^{(\ell)})^2 \leq \frac{\tr\lr{G^{(\ell+1)}}}{\tr(G^{(\ell)})}\leq \frac{1}{2}C_W^{(\ell+1)}n_{\ell}\lambda_{\text{max}}(\mP^{(\ell)})^2.
\end{equation}
\end{proposition}
\begin{proof}
For any fixed $u_1,u_2\in V$ we have 
\begin{align*}
    K_{u_1u_2}^{(\ell+1)}&=\E{z_{u_1;i}^{(\ell+1)}z_{u_2;i}^{(\ell+1)}}\\
    &=\E{\sum_{v_1,v_2\in V}\sum_{j_1,j_2=1}^{n_\ell} W_{ij_1}^{(\ell+1)}W_{ij_2}^{(\ell+1)}\mP_{u_1v_1}^{(\ell)}\mP_{u_2v_2}^{(\ell)}\sigma\lr{z_{v_1;j}^{(\ell)}}\sigma\lr{z_{v_2;j}^{(\ell)}}}\\
    &=\sum_{v_1,v_2\in V} \mP_{u_1v_1}^{(\ell)}\mP_{u_2v_2}^{(\ell)} \E{C_W^{(\ell)}\sum_{j=1}^{n_\ell} \sigma\lr{z_{v_1;j}^{(\ell)}}\sigma\lr{z_{v_2;j}^{(\ell)}}}\\
    &=C_W^{(\ell)}n_\ell \lr{\mP^{(\ell)}\E{\sigma\lr{z_j^{(\ell)}}\sigma\lr{z_j^{(\ell)}}^T}\mP^{(\ell)}}_{u_1,u_2}.
\end{align*}
This proves \eqref{eq:K-rec}. In particular, we have
\begin{align*}
    \tr(K^{(\ell+1)})=C_W^{(\ell)}n_\ell\E{\frac{1}{n_\ell}\sum_{j=1}^{n_\ell}\norm{\mP^{(\ell)}\sigma\lr{z_j^{(\ell)}}}^2}\leq C_W^{(\ell)}n_\ell\lambda_1(\mP^{(\ell)})^2\E{\frac{1}{n_\ell}\sum_{j=1}^{n_\ell}\sum_{v\in V}{\sigma(z_{v;j}^{(\ell)})}^2}.
\end{align*}
Using Lemma \ref{lem:half} we obtain 
\[
\tr(K^{(\ell+1)})\leq\frac{1}{2}C_W^{(\ell)}n_\ell\lambda_1(\mP^{(\ell)})^2 \E{\frac{1}{n_\ell}\sum_{j=1}^{n_\ell}\sum_{v\in V}\lr{z_{v;j}^{(\ell)}}^2}=\frac{1}{2}C_W^{(\ell)}n_\ell\lambda_1(\mP^{(\ell)})^2\tr\lr{K^{(\ell)}}.
\]
This yields the upper bound in \eqref{eq:tr-bounds-K}. The lower bound is obtained in the same way as the upper bound. The derivation of both the recursion \eqref{eq:G-rec} for entries of $G^{(\ell+1)}$ and of the trace estimates \eqref{eq:tr-bounds-G} are obtained in the same way as for $K^{(\ell+1)}$.
\end{proof}

It now remains to obtain the lower bounds in \eqref{eq:M-est}. We give the details when $M=K$ since the proof when $M=G$ the same. Lemma \ref{lem:pi-pos} allows us to choose an orthonormal basis for the top eigenspace of $\mP^{(\ell)}:$
\[
\Pi_1 = \sum_{j=1}^{d_1^{(\ell)}} \pi_j\pi_j^T,\qquad d_1^{(\ell)} := \mathrm{dim}\lr{\mP^{(\ell)}-\lambda_1(\mP^{(\ell)})I}
\]
in which the components are non-negative:
\[
\pi_j = \lr{(\pi_j)_v,\,v\in V},\qquad (\pi_{j})_v\geq 0\,\, \forall v\in V,\, j=1,\ldots, d_1^{(\ell)}.
\]
We have
\[
\tr(K^{(\ell+1)}) \geq \tr(K^{(\ell+1)}\Pi_1)=\sum_{j=1}^{d_1^{(\ell)}} \pi_j^TK^{(\ell+1)}\pi_j. 
\]
Recall we've assumed in \eqref{eq:P-as2} that the entries of $\mP^{(\ell)}$ are non-negative. Hence, for each $j=1,\ldots, d_1^{(\ell)}$ applying \eqref{eq:K-rec} gives
\begin{align*}
   \pi_j^TK^{(\ell+1)}\pi_j&=C_W^{(\ell+1)}n_\ell\pi_j^T \mP^{(\ell)} \E{\sigma(z_{i}^{(\ell)})\sigma(z_i^{(\ell)})}\lr{\mP^{(\ell)}}^T \pi_j\\
    &=\frac{1}{2}C_W^{(\ell+1)} \lr{\lambda_{1}(\mP^{(\ell)})}^2 \sum_{u_1,u_2} (\pi_j)_{u_1}(\pi_j)_{u_2} 2\E{\sigma(z_{u_1;i}^{(\ell)})\sigma(z_{u_2;i}^{(\ell)})}.
\end{align*}
Conditioning on $z^{(\ell-1)}$ and then applying Lemma \ref{lem:half} yields
\begin{align*}
   \pi_j^TK^{(\ell+1)}\pi_j&\geq\frac{1}{2}C_W^{(\ell+1)} \lr{\lambda_{1}(\mP^{(\ell)})}^2 \sum_{u_1,u_2} (\pi_j)_{u_1}(\pi_j)_{u_2} K_{u_1u_2}^{(\ell)}=\frac{1}{2}C_W^{(\ell+1)}n_\ell \lr{\lambda_{1}(\mP^{(\ell)})}^2\pi_j^T K^{(\ell)}\pi_j.
\end{align*}
Summing this inequality over $j$ completes the proof of \eqref{eq:M-est} for $M=K$.\hfill $\square$

\subsection{Proof of Theorem \ref{thm:oversmoothing}}\label{sec:oversmoothing-pf}
Recall the spectral decomposition
\[
\mP^{(\ell)} = \sum_{j=1}^{\abs{V}}\lambda_j(\mP^{(\ell)}) \Pi_j ,\qquad \lambda_1(\mP^{(\ell)})> \cdots> \lambda_{k}(\mP^{(\ell)}),
\]
where $\lambda_{j}(\mP^{(\ell)})$ are the distinct eigenvalues of $\mP^{(\ell)}$ and $\Pi_j$ is the orthogonal projection onto the corresponding eigenspace. We start by deriving the upper bound  \eqref{eq:os-rate-ub}. The upper bound in the estimate \eqref{eq:M-est} reads
\[
\tr(K^{(\ell+1)}\Pi_1)\leq \lr{\prod_{\ell'=1}^\ell \frac{1}{2}C_W^{(\ell'+1)} n_{\ell'}\lambda_{\text{max}}(\mP^{(\ell')})^2 }\tr(K^{(1)}\Pi_1).
\]
Further, the lower bound in \eqref{eq:tr-bounds-K} yields
\[
\tr(K^{(\ell+1)})\geq \lr{\prod_{\ell'=1}^\ell \frac{1}{2}C_W^{(\ell'+1)} n_{\ell'}\lambda_{\text{min}}(\mP^{(\ell')})^2 }\tr(K^{(1)}).
\]
Combining these two immediately gives the upper bound \eqref{eq:os-rate-ub}. To obtain lower bound let us write
\[
\Pi:=I-\Pi_1=\sum_{j=2}^{k}\Pi_j
\]
for the projection onto the orthogonal complement of the top eigenspace of $\mP^{(\ell)}$. The lower bound \eqref{eq:os-rate-lb} will follow readily from the following two estimates:
\begin{align}
\label{eq:os-goal1}    \tr(K^{(\ell+1)}\Pi_1) &\geq \lr{\prod_{\ell'=1}^\ell\frac{ 1}{2}C_W^{(\ell'+1)}n_{\ell'} \lambda_1(\mP^{(\ell')})^2 }\tr(K^{(1)}\Pi_1)\\
\label{eq:os-goal2}    \tr\lr{K^{(\ell+1)}\Pi} &\leq  \lr{\prod_{\ell'=1}^\ell \frac{1}{2}C_W^{(\ell'+1)}n_{\ell'}\lambda_2(\mP^{(\ell')})^2} \tr\lr{K^{(1)}\Pi}.
\end{align}
The inequality \eqref{eq:os-goal1} is simply the lower bound in \eqref{eq:M-est}. To show \eqref{eq:os-goal2}, we use \eqref{eq:K-rec} to write
\begin{align}
\label{eq:Pi-tr}    \tr(K^{(\ell+1)}\Pi)& = C_W^{(\ell+1)}n_\ell \E{\norm{\Pi \mP^{(\ell)} \sigma(z_j^{(\ell)})}_2^2}\leq \frac{1}{2}C_W^{(\ell+1)}n_\ell \lambda_2(\mP^{(\ell)})^2\E{2\norm{\Pi \sigma(z_j^{(\ell)})}^2}.
\end{align}
Let us write
\[
\pi_j,\quad j=1,\ldots, \abs{V}-d_1^{(\ell)}
\]
for an orthonormal basis spanning the orthogonal complement to the top eigenspace of $\mP^{(\ell)}$. We have
\begin{align}
\notag    2\E{\norm{\Pi \sigma(z_j^{(\ell)})}^2}&= 2\sum_{j=1}^{\abs{V}-d_1^{(\ell)}} \E{(\pi_j^T\sigma(z_{i}^{(\ell)}))^2}\\
    &= 2 \sum_{v_1,v_2\in V}\left\{\sum_{j=2}^{\abs{V}-d_1^{(\ell)}} (\pi_j)_{v_1}(\pi_j)_{v_2}\right\} \E{\sigma(z_{v_1;i}^{(\ell)})\sigma(z_{v_2;i}^{(\ell)})}\\
\notag    &= 2 \sum_{v\in V}\left\{\sum_{j=2}^{\abs{V}-d_1^{(\ell)}} (\pi_j)_{v}^2\right\} \E{\sigma(z_{v;i})^2}\\
\label{eq:Pi-est}    &+ 2 \sum_{\substack{v_1,v_2\in V\\ v_1\neq v_2}}\left\{\sum_{j=2}^{\abs{V}-d_1^{(\ell)}} (\pi_j)_{v_1}(\pi_j)_{v_2}\right\} \E{\sigma(z_{v_1;j}^{(\ell)})\sigma(z_{v_2;j}^{(\ell)})}.
\end{align}
Note that by Lemma \ref{lem:half}
\[
\E{\sigma(z_{v;i}^{(\ell)})^2}= \frac{1}{2}K_{vv}^{(\ell)}.
\]
Moreover, by Lemma \ref{L:interference}, we have that 
\[
\sup_{\substack{v_1,v_2\in V\\ v_1\neq v_2}}\sum_{j=2}^{\abs{V}-d_1^{(\ell)}} (\pi_j)_{v_1}(\pi_j)_{v_2}< 0. 
\]
Thus, Lemma \ref{lem:sigma-trick} gives 
\[
2 \sum_{\substack{v_1,v_2\in V\\ v_1\neq v_2}}\left\{\sum_{j=2}^{\abs{V}} (\pi_j)_{v_1}(\pi_j)_{v_2}\right\} \E{\sigma(z_{v_1;j}^{(\ell)})\sigma(z_{v_2;j}^{(\ell)})} \leq \sum_{\substack{v_1,v_2=1\\ v_1\neq v_2}}^{\abs{V}}\left\{\sum_{j=2}^{\abs{V}} (\pi_j)_{v_1}(\pi_j)_{v_2}\right\} K_{v_1v_2}^{(\ell)}.
\]
Substituting this estimate into \eqref{eq:Pi-est} yields
\begin{align*}
     2\E{\norm{\Pi \sigma(z_j^{(\ell)})}^2}&\leq \tr(K^{(\ell)}\Pi).
\end{align*}
This gives
\[
\tr\lr{K^{(\ell+1)}\Pi} \leq   \frac{1}{2}C_W^{(\ell+1)}n_{\ell}\lambda_2(\mP^{(\ell)})^2 \tr\lr{K^{(\ell)}\Pi},
\]
and iterating the inequality proves \eqref{eq:os-goal2}. Finally, combining \eqref{eq:os-goal1} with \eqref{eq:os-goal2} gives
\begin{align}
\notag    r^{(\ell+1)}&=\frac{\tr( K^{(\ell+1)}\Pi_1)}{\tr(K^{(\ell+1)})}\\
\notag &=\frac{\tr( K^{(\ell+1)}\Pi_1)}{\tr( K^{(\ell+1)}\Pi_1) + \tr(K^{(\ell+1)}\Pi)}\\
\notag    &=\lr{1+ \frac{\tr(K^{(\ell+1)}\Pi)}{\tr( K^{(\ell+1)}\Pi_1)}}^{-1}\\
\notag &\geq \lr{1+\lr{\prod_{\ell'=1}^\ell \frac{\lambda_2(\mP^{(\ell')})}{\lambda_1(\mP^{(\ell')})}}^2\frac{\tr(K^{(1)}\Pi)}{\tr( K^{(\ell+1)}\Pi_1)}}^{-1}\\
\label{eq:prod-lb}    &=\lr{1+\lr{\prod_{\ell'=1}^\ell \frac{\lambda_2(\mP^{(\ell')})}{\lambda_1(\mP^{(\ell')})}}^2\lr{\frac{1}{r^{(1)}}-1}}^{-1}
\end{align}
which yields the lower bound in \eqref{eq:os-rate-lb}. 
\hfill $\square$

\subsection{Proof of Corollary \ref{cor:res-os}}\label{sec:res-os-pf}
To derive \eqref{eq:os-lb} and \eqref{eq:os-ub}, note that with 
\[
P^{(\ell)} = \lr{1-t_\ell}I + t_\ell \mP,\qquad \inf_\ell\left\{ 1+t_\ell(\lambda_{\abs{V}}(\mP)-1) \right\}= \delta>0
\]
we have 
\begin{align*}
    \frac{\lambda_2(P^{(\ell)})}{\lambda_1(P^{(\ell)})} &= \frac{1+(\lambda_2(\mP)-1)t_\ell}{1+(\lambda_1(\mP)-1)t_\ell}=1 + \frac{\lambda_2(\mP)-\lambda_1(\mP)}{1+(\lambda_1(\mP)-1)t_\ell}t_\ell.\\
\end{align*}
Since $t_\ell\in [0,1]$ we further have
\[
1+(\lambda_1(\mP)-1)t_\ell \leq\max\set{1,\lambda_1(\mP)}.
\]
Thus, since $\lambda_2(\mP)-\lambda_1(\mP)<0$, we find
\begin{align*}
    \frac{\lambda_2(P^{(\ell)})}{\lambda_1(P^{(\ell)})} &\leq 1 + \frac{\lambda_2(\mP)-\lambda_1(\mP)}{\max\set{1,\lambda_1(\mP)}}t_\ell.\\
\end{align*}
Substituting this into \eqref{eq:prod-lb} yields \eqref{eq:os-lb}. Similarly, 
\begin{align*}
    \frac{\lambda_1(P^{(\ell)})}{\lambda_{\abs{V}}(P^{(\ell)})} &= \frac{1+(\lambda_1(\mP)-1)t_\ell}{1+(\lambda_{\abs{V}}(\mP)-1)t_\ell}=1 + \frac{\lambda_1(\mP)-\lambda_{\abs{V}}(\mP)}{1+(\lambda_{\abs{V}}(\mP)-1)t_\ell}t_\ell.
\end{align*}
Again using that $t_\ell\in [0,1]$, we find
\[
1+(\lambda_{\abs{V}}(\mP)-1)t_\ell\geq \min\set{1,\lambda_{\abs{V}}(\mP)}.
\]
substituting this into the upper bound \eqref{eq:os-rate-ub} yields the upper bounds  \eqref{eq:os-ub}.

\bibliography{example_paper.bib}

\newcommand{\etalchar}[1]{$^{#1}$}
\begin{thebibliography}{XBSD{\etalchar{+}}18}

\bibitem[Abb18]{abbe2018community}
Emmanuel Abbe.
\newblock Community detection and stochastic block models: Recent developments.
\newblock {\em Journal of Machine Learning Research}, 18(177):1--86, 2018.

\bibitem[ABH16]{abbe2016exact}
Emmanuel Abbe, Afonso~S. Bandeira, and Georgina Hall.
\newblock Exact recovery in the stochastic block model.
\newblock {\em IEEE Transactions on Information Theory}, 62(1):471--487, 2016.

\bibitem[BBCV21]{bronstein2021geometric}
Michael~M Bronstein, Joan Bruna, Taco Cohen, and Petar Veli{\v{c}}kovi{\'c}.
\newblock Geometric deep learning: Grids, groups, graphs, geodesics, and
  gauges.
\newblock {\em arXiv preprint arXiv:2104.13478}, 2021.

\bibitem[BSF94]{bengio1994learning}
Yoshua Bengio, Patrice Simard, and Paolo Frasconi.
\newblock Learning long-term dependencies with gradient descent is difficult.
\newblock {\em IEEE transactions on neural networks}, 5(2):157--166, 1994.

\bibitem[BVH16]{bandeira2016sharp}
Afonso~S Bandeira and Ramon Van~Handel.
\newblock Sharp nonasymptotic bounds on the norm of random matrices with
  independent entries.
\newblock {\em Annals of Probability (to appear)}, 2016.

\bibitem[BY08]{bai2008limit}
Zhi-Dong Bai and Yong-Qua Yin.
\newblock Limit of the smallest eigenvalue of a large dimensional sample
  covariance matrix.
\newblock In {\em Advances In Statistics}, pages 108--127. World Scientific,
  2008.

\bibitem[CLL{\etalchar{+}}20]{chen2020measuring}
Deli Chen, Yankai Lin, Wei Li, Peng Li, Jie Zhou, and Xu~Sun.
\newblock Measuring and relieving the over-smoothing problem for graph neural
  networks from the topological view.
\newblock In {\em Proceedings of the AAAI conference on artificial
  intelligence}, volume~34, pages 3438--3445, 2020.

\bibitem[CW20]{cai2020note}
Chen Cai and Yusu Wang.
\newblock A note on over-smoothing for graph neural networks.
\newblock {\em arXiv preprint arXiv:2006.13318}, 2020.

\bibitem[CWH{\etalchar{+}}20]{chen2020simple}
Ming Chen, Zhewei Wei, Zengfeng Huang, Bolin Ding, and Yaliang Li.
\newblock Simple and deep graph convolutional networks.
\newblock In {\em International conference on machine learning}, pages
  1725--1735. PMLR, 2020.

\bibitem[DYZ23]{dinan2023effective}
Emily Dinan, Sho Yaida, and Susan Zhang.
\newblock Effective theory of transformers at initialization.
\newblock {\em arXiv preprint arXiv:2304.02034}, 2023.

\bibitem[FL19]{fey2019fast}
Matthias Fey and Jan~Eric Lenssen.
\newblock Fast graph representation learning with pytorch geometric.
\newblock 2019.

\bibitem[GSG{\etalchar{+}}21]{godwin2021simple}
Jonathan Godwin, Michael Schaarschmidt, Alexander Gaunt, Alvaro
  Sanchez-Gonzalez, Yulia Rubanova, Petar Veli{\v{c}}kovi{\'c}, James
  Kirkpatrick, and Peter Battaglia.
\newblock Simple gnn regularisation for 3d molecular property prediction \&
  beyond.
\newblock {\em arXiv preprint arXiv:2106.07971}, 2021.

\bibitem[Han18]{hanin2018neural}
Boris Hanin.
\newblock Which neural net architectures give rise to exploding and vanishing
  gradients?
\newblock In {\em Advances in Neural Information Processing Systems}, 2018.

\bibitem[Han22]{hanin2022random}
Boris Hanin.
\newblock Random fully connected neural networks as perturbatively solvable
  hierarchies.
\newblock {\em arXiv preprint arXiv:2204.01058}, 2022.

\bibitem[HLD{\etalchar{+}}21]{huang2021towards}
Wei Huang, Yayong Li, Weitao Du, Jie Yin, Richard~Yi Da~Xu, Ling Chen, and Miao
  Zhang.
\newblock Towards deepening graph neural networks: A gntk-based optimization
  perspective.
\newblock {\em arXiv preprint arXiv:2103.03113}, 2021.

\bibitem[HN20]{hanin2020products}
Boris Hanin and Mihai Nica.
\newblock Products of many large random matrices and gradients in deep neural
  networks.
\newblock {\em Communications in Mathematical Physics}, 376(1):287--322, 2020.

\bibitem[HR18]{hanin2018start}
Boris Hanin and David Rolnick.
\newblock How to start training: The effect of initialization and architecture.
\newblock In {\em Advances in Neural Information Processing Systems}, pages
  571--581, 2018.

\bibitem[HZRS15]{he2015delving}
Kaiming He, Xiangyu Zhang, Shaoqing Ren, and Jian Sun.
\newblock Delving deep into rectifiers: Surpassing human-level performance on
  imagenet classification.
\newblock In {\em Proceedings of the IEEE international conference on computer
  vision}, pages 1026--1034, 2015.

\bibitem[JWC{\etalchar{+}}22]{jaiswal2022old}
Ajay Jaiswal, Peihao Wang, Tianlong Chen, Justin Rousseau, Ying Ding, and
  Zhangyang Wang.
\newblock Old can be gold: Better gradient flow can make vanilla-gcns great
  again.
\newblock {\em Advances in Neural Information Processing Systems},
  35:7561--7574, 2022.

\bibitem[Ker22]{keriven2022not}
Nicolas Keriven.
\newblock Not too little, not too much: a theoretical analysis of graph (over)
  smoothing.
\newblock {\em arXiv preprint arXiv:2205.12156}, 2022.

\bibitem[KW16]{kipf2016semi}
Thomas~N Kipf and Max Welling.
\newblock Semi-supervised classification with graph convolutional networks.
\newblock {\em arXiv preprint arXiv:1609.02907}, 2016.

\bibitem[LGJ20]{liu2020towards}
Meng Liu, Hongyang Gao, and Shuiwang Ji.
\newblock Towards deeper graph neural networks.
\newblock In {\em Proceedings of the 26th ACM SIGKDD international conference
  on knowledge discovery \& data mining}, pages 338--348, 2020.

\bibitem[LHW18]{li2018deeper}
Qimai Li, Zhichao Han, and Xiao-Ming Wu.
\newblock Deeper insights into graph convolutional networks for semi-supervised
  learning.
\newblock In {\em Proceedings of the AAAI conference on artificial
  intelligence}, volume~32, 2018.

\bibitem[LMGK21]{li2021training}
Guohao Li, Matthias M{\"u}ller, Bernard Ghanem, and Vladlen Koltun.
\newblock Training graph neural networks with 1000 layers.
\newblock In {\em International conference on machine learning}, pages
  6437--6449. PMLR, 2021.

\bibitem[LMQ{\etalchar{+}}21]{li2021deepgcns}
Guohao Li, Matthias M{\"u}ller, Guocheng Qian, Itzel Carolina~Delgadillo Perez,
  Abdulellah Abualshour, Ali~Kassem Thabet, and Bernard Ghanem.
\newblock Deepgcns: Making gcns go as deep as cnns.
\newblock {\em IEEE Transactions on Pattern Analysis and Machine Intelligence},
  2021.

\bibitem[LZG{\etalchar{+}}21]{lu2021skipnode}
Weigang Lu, Yibing Zhan, Ziyu Guan, Liu Liu, Baosheng Yu, Wei Zhao, Yaming
  Yang, and Dacheng Tao.
\newblock Skipnode: On alleviating over-smoothing for deep graph convolutional
  networks.
\newblock {\em arXiv preprint arXiv:2112.11628}, 2021.

\bibitem[MBJ{\etalchar{+}}23]{musaelian2023learning}
Albert Musaelian, Simon Batzner, Anders Johansson, Lixin Sun, Cameron~J Owen,
  Mordechai Kornbluth, and Boris Kozinsky.
\newblock Learning local equivariant representations for large-scale atomistic
  dynamics.
\newblock {\em Nature Communications}, 14(1):579, 2023.

\bibitem[OS19]{oono2019graph}
Kenta Oono and Taiji Suzuki.
\newblock Graph neural networks exponentially lose expressive power for node
  classification.
\newblock {\em arXiv preprint arXiv:1905.10947}, 2019.

\bibitem[PLR{\etalchar{+}}16]{poole2016exponential}
Ben Poole, Subhaneil Lahiri, Maithra Raghu, Jascha Sohl-Dickstein, and Surya
  Ganguli.
\newblock Exponential expressivity in deep neural networks through transient
  chaos.
\newblock In {\em Advances in neural information processing systems}, pages
  3360--3368, 2016.

\bibitem[RPK{\etalchar{+}}17]{schoenholz2016deep}
Maithra Raghu, Ben Poole, Jon~M. Kleinberg, Surya Ganguli, and Jascha
  Sohl{-}Dickstein.
\newblock On the expressive power of deep neural networks.
\newblock In {\em Proceedings of the 34th International Conference on Machine
  Learning, {ICML} 2017}, pages 2847--2854, 2017.

\bibitem[RYH22]{roberts2022principles}
Daniel~A Roberts, Sho Yaida, and Boris Hanin.
\newblock {\em The Principles of Deep Learning Theory: An Effective Theory
  Approach to Understanding Neural Networks}.
\newblock Cambridge University Press, 2022.

\bibitem[Sch]{SchmidtHuber}
Sepp hochreiter's fundamental deep learning problem (1991).
\newblock
  \url{http://people.idsia.ch/~juergen/fundamentaldeeplearningproblem.html}.
\newblock Accessed: 2017-12-26.

\bibitem[SEG21]{sabanayagam2021new}
Mahalakshmi Sabanayagam, Pascal Esser, and Debarghya Ghoshdastidar.
\newblock New insights into graph convolutional networks using neural tangent
  kernels.
\newblock {\em arXiv preprint arXiv:2110.04060}, 2021.

\bibitem[SGGSD17]{schoenholz2016deepinformation}
Samuel~S Schoenholz, Justin Gilmer, Surya Ganguli, and Jascha Sohl-Dickstein.
\newblock Deep information propagation.
\newblock {\em ICLR 2017 and arXiv:1611.01232}, 2017.

\bibitem[WCWJ22]{wu2022non}
Xinyi Wu, Zhengdao Chen, William Wang, and Ali Jadbabaie.
\newblock A non-asymptotic analysis of oversmoothing in graph neural networks.
\newblock {\em arXiv preprint arXiv:2212.10701}, 2022.

\bibitem[WPC{\etalchar{+}}20]{wu2020comprehensive}
Zonghan Wu, Shirui Pan, Fengwen Chen, Guodong Long, Chengqi Zhang, and S~Yu
  Philip.
\newblock A comprehensive survey on graph neural networks.
\newblock {\em IEEE transactions on neural networks and learning systems},
  32(1):4--24, 2020.

\bibitem[XBSD{\etalchar{+}}18]{xiao2018dynamical}
Lechao Xiao, Yasaman Bahri, Jascha Sohl-Dickstein, Samuel~S Schoenholz, and
  Jeffrey Pennington.
\newblock Dynamical isometry and a mean field theory of cnns: How to train
  10,000-layer vanilla convolutional neural networks.
\newblock {\em ICML and arXiv:1806.05393}, 2018.

\bibitem[YHS{\etalchar{+}}22]{yan2022two}
Yujun Yan, Milad Hashemi, Kevin Swersky, Yaoqing Yang, and Danai Koutra.
\newblock Two sides of the same coin: Heterophily and oversmoothing in graph
  convolutional neural networks.
\newblock In {\em 2022 IEEE International Conference on Data Mining (ICDM)},
  pages 1287--1292. IEEE, 2022.

\bibitem[YWY{\etalchar{+}}20]{yang2020revisiting}
Chaoqi Yang, Ruijie Wang, Shuochao Yao, Shengzhong Liu, and Tarek Abdelzaher.
\newblock Revisiting over-smoothing in deep gcns.
\newblock {\em arXiv preprint arXiv:2003.13663}, 2020.

\bibitem[ZCH{\etalchar{+}}20]{zhou2020graph}
Jie Zhou, Ganqu Cui, Shengding Hu, Zhengyan Zhang, Cheng Yang, Zhiyuan Liu,
  Lifeng Wang, Changcheng Li, and Maosong Sun.
\newblock Graph neural networks: A review of methods and applications.
\newblock {\em AI open}, 1:57--81, 2020.

\bibitem[ZDM19]{zhang2019fixup}
Hongyi Zhang, Yann~N Dauphin, and Tengyu Ma.
\newblock Fixup initialization: Residual learning without normalization.
\newblock {\em arXiv preprint arXiv:1901.09321}, 2019.

\bibitem[ZFMH21]{zheng2021deeper}
Lecheng Zheng, Dongqi Fu, Ross Maciejewski, and Jingrui He.
\newblock Deeper-gxx: deepening arbitrary gnns.
\newblock {\em arXiv preprint arXiv:2110.13798}, 2021.

\end{thebibliography}
\bibliographystyle{alpha}

\end{document}